\documentclass{article}

\usepackage{arxiv}

\usepackage[utf8]{inputenc} 
\usepackage[T1]{fontenc}    
\usepackage[hidelinks]{hyperref}       
\usepackage{url}            
\usepackage{booktabs}       
\usepackage{amsfonts}       
\usepackage{nicefrac}       
\usepackage{microtype}      
\usepackage{xcolor}         

\usepackage{amsmath,amssymb,amsthm}
\usepackage{algorithm}
\usepackage{caption}
\usepackage{subcaption}
\usepackage[noend]{algorithmic}
\usepackage[utf8]{inputenc}

\newtheorem{Def}{Definition}[section]
\newtheorem{Thm}{Theorem}[section]
\newtheorem{Lem}{Lemma}[section]
\newtheorem{Cor}{Corollary}[section]
\newtheorem{Clm}{Claim}[section]

\DeclareMathOperator*{\argmin}{arg\,min}

\newcommand{\B}{\mathcal B}
\newcommand{\C}{\mathbb C}
\newcommand{\D}{\mathcal D}
\newcommand{\E}{\mathbb E}

\newcommand{\R}{\mathbb R}

\newcommand{\X}{\mathcal X}

\newcommand{\1}{\mathbf 1}
\newcommand{\vol}{\operatorname{vol}}

\usepackage{scalerel,mathtools,color}
\let\svsqrt\sqrt
\newsavebox\Nsqrt
\def\sr#1{\ThisStyle{%
	\savebox\Nsqrt{\scalebox{.5}[1]{$\SavedStyle\svsqrt{\phantom{\cramped{#1#1}}}$}}%
	\ooalign{\usebox{\Nsqrt}\cr\kern.2pt\usebox{\Nsqrt}\cr\hfil$\SavedStyle\cramped{#1}$}}}

\usepackage{enumitem}
\usepackage{bm}
\def\*#1{\mathbf{#1}}

\usepackage{threeparttable,multirow}

\title{Learning-to-learn non-convex\\piecewise-Lipschitz functions\thanks{Author addresses: \texttt{\{ninamf, mkhodak, dravyans, atalwalk\}@cs.cmu.edu}}}

\author{\large\normalfont
  Maria-Florina Balcan
  \And\large\normalfont
  Mikhail Khodak
  \And\large\normalfont
  Dravyansh Sharma
  \And\large\normalfont
  Ameet Talwalkar
  }

\begin{document}

\maketitle


\begin{abstract}

We analyze the meta-learning of the initialization and step-size of learning algorithms for piecewise-Lipschitz functions, a non-convex setting with applications to both machine learning and algorithms. 
Starting from recent regret bounds for the exponential forecaster on losses with dispersed discontinuities, we generalize them to be initialization-dependent and then use this result to propose a practical meta-learning procedure that learns both the initialization and the step-size of the algorithm from multiple online learning tasks. 
Asymptotically, we guarantee that the average regret across tasks scales with a natural notion of task-similarity that measures the amount of overlap between near-optimal regions of different tasks. 
Finally, we instantiate the method and its guarantee in two important settings:
robust meta-learning and multi-task data-driven algorithm design.

\end{abstract}

\section{Introduction}\label{sec:intro}

While learning-to-learn, or {\em meta-learning}, has long been an object of study \cite{thrun1998ltl}, in recent years it has gained significant attention as a multi-task paradigm for developing algorithms for learning in dynamic environments, from multiple sources of data, and in federated settings.
Such methods focus on using data gathered from multiple tasks to improve performance when faced with data from a new, potentially related task.
Among the more popular approaches to meta-learning is {\em initialization-based} meta-learning, in which the meta-learner uses multi-task data to output an initialization for an iterative algorithm such as stochastic gradient descent (SGD) \cite{finn2017maml}.
The flexibility of this approach has led to its widespread adoption in areas, such as robotics \cite{duan2017imitation} and federated learning \cite{chen2018fedmeta}, and to a growing number of attempts to understand it, both empirically and theoretically \cite{denevi2019ltlsgd,khodak2019adaptive,fallah2020meta,raghu2020anil,saunshi2020meta}.
However, outside some stylized setups our learning-theoretic understanding of how to meta-learn an initialization is largely restricted to the convex Lipschitz setting.

We relax both assumptions to study the meta-learning of online algorithms over piecewise-Lipschitz functions, which can be nonconvex and highly discontinuous.
As no-regret online learning over such functions is impossible in-general, we study the case of piecewise-Lipschitz functions whose discontinuities are {\em dispersed}, i.e. which do not concentrate in any small compact subset of the input domain \cite{balcan2018dispersion}.
Such functions arise frequently in {\em data-driven algorithm design}, in which the goal is to learn the optimal parameter settings of algorithms for difficult (often NP-Hard) problems over a distribution or sequence of instances \cite{balcan2020data};
for example, a small change to the metric used to determine cluster linkage can lead to a discontinuous change in the classification error \cite{balcan2019learning}.
In this paper, we also demonstrate that such losses are relevant in the setting of adversarial robustness, where we introduce a novel online formulation.
For both cases, the associated problems are often solved across many time periods or for many different problem domains, resulting in natural multi-task structure that we might hope to use to improve performance.
To the best of our knowledge, ours is the first theoretical study of meta-learning in both of these application settings.

In the single-task setting the problem of learning dispersed functions can be solved using simple methods such as the exponentially-weighted forecaster.
To design an algorithm for learning to initialize online learners in this setting, we propose a method that optimizes a sequence of data-dependent upper-bounds on the within-task regret \cite{khodak2019adaptive}. The result is an averaged bound that improves upon the regret of the single-task exponential forecaster so long as there exists an initial distribution that can compactly contain many of the within-task optima of the different tasks.
Designing the meta-procedure is especially challenging in our setting because it involves online learning over a set of distributions on the domain.
To handle this we study a ``prescient'' form of the classic follow-the-regularized leader (FTRL) scheme that is run over an unknown discretization;
we then show the existence of another algorithm that plays the same actions but uses only known information, thus attaining the same regret while being practical to implement.

To demonstrate the usefulness of our method, we study this algorithm in two settings.

{\bf Multi-task data-driven algorithm design.} We consider data-driven tuning of the parameters of combinatorial optimization algorithms for hard problems such as knapsack and clustering.
The likely intractability of these problems have led to several approaches to study them in more realistic settings, such as smoothed analysis \cite{spielman2004smoothed} and data-driven algorithm configuration \cite{balcan2020data}.
We view our meta-learning approach as a refinement on the latter in which we allow not only a distribution of instances but multiple distributions of related instances that can help learn a good algorithm. Our setting is more realistic than those considered in prior work. It is more challenging than learning from i.i.d. instances \cite{gupta2017pac,balcan2017learning}, but at the same time less pessimistic than online learning over adversarial problem instances~\cite{balcan2018dispersion}, as it allows us to leverage similarity of problem instances coming from different but related distributions.
We instantiate our bounds theoretically on several problems where the cost functions are piecewise-constant in the tuned parameters, allowing our meta-procedure to learn the right initial distribution for exponential forecasters. This includes well-known combinatorial optimization problems like finding the maximum weighted independent set (MWIS) of vertices on a graph, solving quadratic programs with integer constraints using algorithms based on the celebrated Goemans-Williamson algorithm, and mechanism design for combinatorial auctions.
Then we consider experimentally the problem of tuning the right $\alpha$ for the $\alpha$-Lloyd's family of clustering algorithms~\cite{balcan2018data}.
In experimental evaluations on two datasets---a synthetic Gaussian mixture model and the well-known Omniglot dataset from meta-learning \cite{lake2015human}---our meta-procedure leads to improved clustering accuracy compared to single-task learning to cluster. The results holds for both one-shot and five-shot clustering tasks.
We also study our results for a family of greedy algorithms for the knapsack problem introduced by \cite{gupta2017pac} and obtain similar results for a synthetic dataset.

{\bf Online robust meta-learning.} The second instantiation of our meta-learning procedure is to a new notion of adversarial robustness for the setting of online learning, where our results imply robust meta-learning in the presence of outliers. 
In this setting, the adversary can make (typically small) modifications to some example $x\in\X$, which can result in potentially large changes to the corresponding loss value $l_h(x)$, where $h\in\mathcal{H}$ is our hypothesis. For instance, consider the well-studied setting of adversarial examples for classification of images using deep neural networks \cite{nguyen2015deep,brendel2020adversarial}. Given a neural network $f$, the adversary can perturb a datapoint $x$ to a point $x'$, say within a small $L_p$-ball around $x$, such that $f(x)=f(x')$ but the true label of $x'$ does not match $x$, and therefore $l_f(x)\ne l_f(x')$. In general, under the adversarial influence, we observe a {\it perturbed loss} function $\Tilde{l}_h(x)=l_h(x)+a_h(x)$. Typically we are interested in optimizing both the  perturbed loss $\Tilde{l}_h(x)$, i.e. measuring performance relative to optimum for adversarially perturbed losses, and the {\it true loss} $l_h(x)$ (performance on the unobserved, unperturbed loss). For example, in the online learning setting, \cite{agarwal2019online}~consider perturbed loss minimization for linear dynamical systems, while \cite{resler2019adversarial} look at true $\{0,1\}$ loss minimization in the presence of adversarial noise. Our approach ensures that regret for both the perturbed and true loss are small, for piecewise-Lipschitz but dispersed adversaries.

\subsection{Related work}\label{sec:related}

The success of meta-learning has led to significant theoretical effort to understand it.
Most efforts studying initialized-based meta-learning focus on the convex Lipschitz setting \cite{denevi2019ltlsgd,khodak2019provable};
work studying inherently nonconvex modeling approaches instead usually study multi-task representation learning~\cite{balcan2015lifelong,maurer2016mtl,du2021fewshot,tripuraneni2021provable} or target optimization, e.g. stationary point convergence \cite{fallah2020meta}.
An exception is a study of linear models over Gaussian data showing that nonconvexity is critical to meta-learning an initialization that exploits low-rank task structure \cite{saunshi2020meta}.
There is also work extending results from the neural tangent kernel literature to meta-learning \cite{zhou2021meta}, but in this case the objective becomes convex.
On the other hand, we study initializations for learning a class of functions that can be highly non-convex and have numerous discontinuities.
Theoretically, our work uses the Average Regret-Upper-Bound Analysis (ARUBA) strategy~\cite{khodak2019adaptive} for obtaining a meta-update procedure for initializing within-task algorithms, which has been applied elsewhere for privacy \cite{li2020dp} and federated learning \cite{khodak2021fedex};
the main technical advance in our work is in providing the guarantees for it in our setting, which is challenging due to the need to learn over a space of probability measures.

Data-driven configuration is the selection of an algorithm from a parameterized family by learning over multiple problem instances \cite{gupta2017pac,balcan2017learning}. In other words, it is `hyperparameter tuning' with formal guarantees, and has applications to integer programming, clustering, and learning with limited labeled data \cite{balcan2018learning,balcan2019learning,balcan2021data}. In this work, we show how this general approach can be made even more effective by enabling it to adapt to task similarity. We also show applications of our results to robust meta-learning in the presence of outliers in the dataset \cite{pillutla2019robust,kong2020robust}.
While previous work on robust online learning has considered adversaries with bounded perturbation in the online learning setting \cite{agarwal2019online,resler2019adversarial}, our results allow potentially unbounded perturbations, provided the adversary uses a smooth distribution. That is, the adversarial attack can be thought of as a distribution of perturbations, similar to the smoothed analysis approach of \cite{spielman2004smoothed}. In the offline setting, a similar attack is studied in the context of deep network feature-space attacks by \cite{balcan2020power}. We also remark that our formulation has a poisoning aspect, since we do not observe the clean loss $l_h(x)$, which is of particular interest in federated learning \cite{bagdasaryan2020backdoor,tolpegin2020data}. Also, note that unlike the typical applications of data-driven design where optimization is over the dual loss function, i.e. loss as a function of the algorithm parameter for a fixed sample $x\in\X$, here we consider learning loss or confidence functions over the input space $\X$.

\section{Preliminaries and initialization-dependent learning of dispersed functions}\label{sec:dispersion}

In this section we introduce our setup and notation for online learning of piecewise-Lipschitz functions in a multi-task environment.
We then generalize existing results for the single-task setting in order to obtain within-task regret bounds that depend on both the initialization and the task data.
This is critical for both defining a notion of task similarity and devising a meta-learning procedure.

\subsection{Meta-learning setup}

Following past setups \cite{alquier2017lifelong,denevi2019meta,khodak2019adaptive}, for some $T,m>0$ and all $t\in[T]$ and $i\in[m]$ we consider a meta-learner faced with a sequence of $Tm$ loss functions $\ell_{t,i}:C\mapsto[0,1]$ over a compact subset $C\subset\R^d$ that lies within a ball $\B(\rho,R)$ of radius $R$ around some point $\rho\in\R^d$.
Here we used the notation $[n]=\{1,\dots,n\}$.
Before each loss function $\ell_{t,i}$ the meta-learner must pick an element $\rho_{t,i}\in C$ before then suffering a loss or cost $\ell_{t,i}(\rho_{t,i})$.
For a fixed $t$, the subsequence $\ell_{t,1},\dots,\ell_{t,m}$ defines a {\bf task} for which we expect a single element $\rho_t^\ast\in C$ to do well, and thus we will use the {\bf within-task regret} on task $t$ to describe the quantity
\begin{equation}\label{eq:regret}
\*R_{t,m}=\sum_{i=1}^m\ell_{t,i}(\rho_{t,i})-\ell_{t,i}(\rho_t^\ast)\quad\textrm{where}\quad\rho_t^\ast\in\argmin_{\rho\in C}\sum_{i=1}^m\ell_{t,i}(\rho)
\end{equation}
In the single-task setting the goal is usually to show that $R_{t,m}$ is sublinear in $m$, i.e. that the average loss decreases with more rounds.
A key point here is that the functions we consider can have numerous global optima.
In this work we will assume, after going through the $m$ rounds of task $t$, that we have oracle access to a single fixed optimum for $t$, which we will refer to using $\rho_t^\ast$ and use in both our algorithm and to define the task-similarity.
Note that in the types of applications we are interested in---piecewise-Lipschitz functions---the complexity of computing optima scales with the number of discontinuities.
In the important special case of piecewise-constant functions, this dependency becomes logarithmic \cite{cohen2017online}.
Thus this assumption does not affect the usefulness of the result.
 
Our goal will be to improve the guarantees for regret in the single-task case by using information obtained from solving multiple tasks.
In particular, we expect average performance across tasks to improve as we see more tasks; 
to phrase this mathematically we define the {\bf task-averaged regret}
\begin{equation}\label{eq:tar}
\*{\bar R}_{T,m}=\frac1T\sum_{t=1}^T\*R_{t,m}=\frac1T\sum_{t=1}^T\sum_{i=1}^m\ell_{t,i}(\rho_{t,i})-\ell_{t,i}(\rho_t^\ast)
\end{equation}
and claim improvement over single-task learning if in the limit of $T\to\infty$ it is smaller than $\*R_{t,m}$.
Note that for simplicity in this work we assume all tasks have the same number of rounds within-task, but as with past work our results are straightforward to extend to the more general setting.

\subsection{Learning piecewise-Lipschitz functions}

We now turn to our target functions and within-task algorithms for learning them:
piecewise-Lipschitz losses, i.e. functions that are $L$-Lipschitz w.r.t. the Euclidean norm everywhere except on measure zero subsets of the space; 
here they may have arbitrary jump discontinuities so long they still bounded between $[0,1]$.
Apart from being a natural setting of interest due to its generality compared to past work on meta-learning, this class of functions has also been shown to have important applications in data-driven algorithm configuration \cite{balcan2018dispersion};
there these functions represent the cost, e.g. an objective value or time-complexity, of algorithms for difficult problems such as integer programming, auction design, and clustering.

This literature has also shown lower bounds demonstrating that no-regret learning piecewise-Lipschitz function is impossible in general, necessitating assumptions about the sequence.
One such condition is {\em dispersion}, which requires that the discontinuities are not too concentrated.

\newpage
\begin{Def}[\cite{balcan2018dispersion}]\label{def:dis} 
The sequence of random loss functions $\ell_1, \dots,\ell_m$ is said to be $\beta$-{\bf dispersed} with Lipschitz constant $L$ if, for all $m$ and for all $\epsilon\ge m^{-\beta}$, we have that, in expectation over the randomness of the functions, at most
$\tilde{O}(\epsilon m)$ functions (the soft-O notation suppresses dependence on quantities beside $\epsilon,m$ and $\beta$, as well as logarithmic terms)
are not $L$-Lipschitz for any pair of points at distance $\epsilon$ in the domain $\C$. That is, for all $m$ and for all $\varepsilon\ge m^{-\beta}$, 
\begin{equation}
\E\left[
\max_{\begin{smallmatrix}\rho,\rho'\in\C\\\|\rho-\rho'\|_2\le\epsilon\end{smallmatrix}}\|\big\lvert
\{ i\in[m] \mid\ell_i(\rho)-\ell_i(\rho')>L\|\rho-\rho'\|_2\} \big\rvert \right] 
\le \tilde{O}(\epsilon m)
\end{equation}
\end{Def}

Assuming a sequence of $m$ $\beta$-dispersed loss functions and initial distribution $w_1$ set to the uniform distribution over $C$ and optimize the step size parameter, the exponential forecaster presented in Algorithm~\ref{alg:ef} achieves sublinear regret $\tilde{O}(\sqrt{dm\log(Rm)}+(L+1)m^{1-\beta})$.
While this result achieves a no-regret procedure, its lack of dependence on both the task-data and on the chosen initialization makes it difficult to meta-learn.
In the following theorem, we generalize the regret bound for the exponential forecaster to make it data-dependent and hyperparameter-dependent:

\begin{Thm}\label{thm:exp-forc-meta}
	Let $\ell_1,\dots,\ell_m: C \mapsto [0, 1]$ be any sequence of piecewise $L$-Lipschitz functions that are $\beta$-dispersed. Suppose $C \subset \R^d$ is contained in a ball of radius $R$. The exponentially weighted forecaster (Algorithm \ref{alg:ef}) has expected regret $\*R_m\le m\lambda +\frac{\log (1/Z)}{\lambda}+\tilde{O}((L+1)m^{1-\beta})$, where $Z=\frac{\int_{\B(\rho^*,m^{-\beta})}w(\rho)d\rho}{\int_{C}w(\rho)d\rho}$ for  $\rho^*$ the optimal action in hindsight.
\end{Thm}

The proof of this result adapts past analyses of Algorithm \ref{alg:ef};
setting step-size $\lambda$ appropriately recovers the previously mentioned bound.
The new bound is useful due to its explicit dependence on both the initialization $w$ and the optimum in hindsight via the $\log(1/Z)$ term.
Assuming $w$ is a (normalized) distribution, this effectively measures the overlap between the chosen initialization and a small ball around the optimum;
we thus call $$-\log Z=-\log\frac{\int_{\B(\rho^\ast,m^{-\beta})}w(\rho)d\rho}{\int_Cw(\rho)d\rho}$$ the {\bf negative log-overlap} of initialization $w(.)$ with the optimum $\rho^*$.

We also obtain an asymptotic lower bound on the expected regret of any algorithm by extending the argument of \cite{balcan2020learning} to the multi-task setting. We show that for finite $D^*$ we must suffer $\Tilde{\Omega}(m^{1-\beta})$ regret, which limits the improvement we can hope to achieve from task-similarity.

\begin{Thm}\label{thm:dispersion-lb}
	There is a sequence of piecewise $L$-Lipschitz $\beta$-dispersed functions $\ell_{i,j}: [0,1] \mapsto [0, 1]$,  whose optimal actions in hindsight $\argmin_{\rho}\sum_{i=1}^ml_{t,i}(\rho)$ are contained in some fixed ball of diameter $D^*$, for which any algorithm has expected regret $\*R_m\ge \tilde{\Omega}(m^{1-\beta})$.
\end{Thm}

\begin{algorithm}[!t]
	\caption{Exponential Forecaster}
	\label{alg:ef}
	\begin{algorithmic}[1]
		\STATE {\bfseries Input:} step size parameter $\lambda \in (0, 1]$, initialization $w:C\rightarrow \R_{\ge 0}$.
		\STATE{Initialize $w_1=w$}
		\FOR{$i=1,2,\dots,m$}
		\STATE{$W_i:=\int_{C}w_i(\rho)d\rho$}
		\STATE{Sample $\rho_i$ with probability proportional to $w_i(\rho_i)$, i.e. with probability
			$p_{i}(\rho_i)=\frac{w_i(\rho_i)}{W_i}$}
		\STATE{Suffer $\ell_i(\rho_i)$ and observe $\ell_i(\cdot)$}
		\STATE{For each $\rho\in C, \text{ set }w_{i+1}(\rho)=e^{-\lambda\ell_i(\rho)}w_{i}(\rho)$}
		\ENDFOR
	\end{algorithmic}
\end{algorithm}

\subsection{Task-similarity}

Before proceeding to our discussion of meta-learning, we first discuss what we might hope to achieve with it;
specifically, we consider what a reasonable notion of task-similarity is in this setting.
Note that the Theorem~\ref{thm:exp-forc-meta} regret bound has three terms, of which two depend on the hyperparameters and the last is due to dispersion and cannot be improved via better settings.
Our focus will thus be on improving the first two terms, which are the dominant ones due to the dependence on the dimensionality and the distance from the initialization encoded in the negative log overlap.
In particular, when the initialization is the uniform distribution then this quantity depends inversely on the size of a small ball around the optimum, which may be quite small.
Via meta-learning we hope to assign more of the probability mass of the initializer to areas close to the optimum, which will decrease these terms.
On average, rather than a dependence on the volume of a small ball we aim to achieve a dependence on the {\bf average negative log-overlap} 
\begin{equation}\label{eq:tasksim}
V^2=-\min_{w:C\mapsto\R_{\ge0},\int_Cw(\rho)d\rho=1}\frac1T\sum_{t=1}^T\log\int_{\B(\rho_t^\ast,m^{-\beta})}w(\rho)d\rho
\end{equation}
which can be much smaller if the task optima $\rho_t^\ast$ are close together;
for example, if they are the same then $V=0$, corresponding to assigning all the initial weight within the common ball $\B(\rho^\ast,m^{-\beta})$ around the shared optima. This is also true if $\vol(\cap_{t\in T}\B(\rho_t^\ast,m^{-\beta}))>0$, as one can potentially initialize with all the weight in the intersection of the balls. On the other hand if $\vol(\cap_{t\in T}\B(\rho_t^\ast,m^{-\beta}))=0$, $V>0$. For example, if a $p$-fraction of tasks have optima $\rho_0$ and the remaining at $\rho_1$ with $||\rho_0-\rho_1||>2m^{-\beta}$ the task similarity is given by the binary entropy function $V=H_b(p)=-p\log p-(1-p)\log(1-p)$.

The settings of Algorithm~\ref{alg:ef} that achieve the minimum in the definition of $V$ are directly related to $V$ itself: the optimal initializer is the distribution achieving $V$ and the optimal step-size is $V/\sqrt m$.
Note that while the explicit definition requires computing a minimum over a set of functions, the task-similarity can be computed using the discretization constructed in Section~\ref{sec:meta}.

\section{An algorithm for meta-learning the initialization and step-size}

Having established a single-task algorithm and shown how its regret depends on the initialization and step-size, we move on to meta-learning these hyperparameters.
Recall that our goal is to make the task-averaged regret \eqref{eq:tar} small, in particular to improve upon the baseline of repeatedly running Algorithm~\ref{alg:ef} from the uniform distribution, up to $o_T(1)$ terms that vanish as we see more tasks.
This accomplishes the meta-learning goal of using multiple tasks to improve upon single-task learning.

In this paper, we use the strategy of running online learning algorithms on the data-dependent regret guarantees from above \cite{khodak2019adaptive}.
If we can do so with sublinear regret in $T$, then we will improve upon the single-task guarantees up to $o_T(1)$ terms, as desired.
Specifically, we are faced with a sequence of regret-upper-bounds $U_t(w,v)=(v+f_t(w)/v)\sqrt m+g(m)$ on nonnegative functions $w$ over $C$ and positive scalars $v>0$.
Note that $g(m)$ cannot be improved via meta-learning, so we will focus on learning $w$ and $v$.
To do so, we run two online algorithms, one over the functions $f_t$ and the other over $h_t(v)=v+f_t(w_t)/v$, where $w_t$ is set by the first procedure.
As shown in the following result, if both procedures have sublinear regret then our task-averaged regret will have the desired properties:
\begin{Thm}\label{lem:aruba}
	Assume each task $t\in[T]$ consists of a sequence of $m$ $\beta$-dispersed piecewise $L$-Lipschitz functions $\ell_{t,i}:C\mapsto[0,1]$.
	Let $f_t$ and $g$ be functions such that the regret of Algorithm~\ref{alg:ef} run with step-size $\lambda=v\sqrt m$ for $v>0$ and initialization $w:C\mapsto\R_{\ge0}$ is bounded by $U_t(w,v)=(v+f_t(w)/v)\sqrt m+g(m)$.
	Suppose we have a procedure that achieves $F_T(w)$ regret w.r.t. any $w:C\mapsto\R_{\ge0}$ by playing actions $w_t:C\mapsto\R_{\ge0}$ on $f_t$ and another procedure that achieves $H_T(v)$ regret w.r.t. any $v>0$ by playing actions $v_t>0$ on $h_t(v)=v+f_t(w_t)/v$, where $H_T$ is non-increasing on the positive reals.
	Then by setting $\rho_{t,i}$ using Algorithm~\ref{alg:ef} with step-size $v_t/\sqrt m$ and initialization $w_t$ at each task $t$ we get task-averaged regret bounded by
	\begin{equation}
	\left(\frac{H_T(V)}T+\min\left\{\frac{F_T(w^\ast)}{VT},2\sqrt{F_T(w^\ast)/T}\right\}+2V\right)\sqrt m+g(m)
	\end{equation}
	for $w^\ast=\argmin_{w:C\mapsto\R_{\ge0}}\sum_{t=1}^Tf_t(w)$ the optimal initialization and $V$ the task-similarity~\eqref{eq:tasksim}.
\end{Thm}

This result is an analog of \cite[Theorem~3.1]{khodak2019adaptive} and follows by manipulating the definition of regret.
It reduces the problem of obtaining a small task-averaged regret to solving two online learning problems, one to set the initialization and one to set the step-size.
So long as both have sublinear regret then we will  improve over single-task learning.
In the next two sections we derive suitable procedures.

\subsection{Meta-learning the initialization}\label{sec:meta}

We now come to the most technically challenging component of our meta-learning procedure:
learning the initialization.
As discussed above, we can accomplish this by obtaining a no-regret procedure for the function sequence $$f_t(w)=-\log\frac{\int_{\B(\rho_t^\ast,m^{-\beta})}w(\rho)d\rho}{\int_Cw(\rho)d\rho}.$$
This is nontrivial as the optimization domain is a set of nonnegative functions, effectively measures on the domain $C$.
To handle this, we first introduce some convenient notation and abstractions.
At each task $t$ we are faced with some function $f_t$ associated with an unknown closed subset $C_t\subset C$ --- in particular $C_t=\B(\rho_t^\ast,m^{-\beta})$ --- with positive volume $\vol(C_t)>0$ that is revealed after choosing $w_t:C\mapsto\R_{\ge0}$.
For each time $t$ define the discretization $$\D_t=\{D=\bigcap_{s\le t}C_s^{(\*c_{[s]})}:\*c\in\{0,1\}^t,\vol(D)>0\}$$ of $C$, where $C_t^{(0)}=C_t$ and $C_t^{(1)}=C\backslash C_t$.
We will use elements of these discretizations to index nonnegative vectors in $\R_{\ge0}^{|\D_t|}$;
specifically, for any measure $w:C\mapsto\R_{\ge0}$ let $\*w(t)\in\R_{\ge0}^{|\D_t|}$ denote the vector with entries $\*w(t)_{[D]}=\int_Dw(\rho)d\rho$ for $D\in\D_t$.
Note that we will exclusively use $p,q,v,w$ for measures, with $v$ specifically referring to the uniform measure, i.e. $\*v(t)_{[D]}=\vol(D)$.
For convenience, for all real vectors $\*x$ we will use $\*{\hat x}$ to denote $\*p/\|\*p\|_1$.
Finally, we abuse notation and remove the parentheses to refer those vectors associated with the final discretization, i.e. $\*v=\*v(T)$ and $\*w=\*w(T)$.

Now that we have this notation we can turn back to the functions we are interested in: 
$f_t(w)=-\log\frac{\int_{C_t}w(\rho)d\rho}{\int_Cw(\rho)d\rho}$, where $C_t=\B(\rho_t^\ast,m^{-\beta})$.
Observe that we can equivalently write this as $f_t(\*w)=-\log\langle\*w_t^\ast,\*{\hat w}\rangle$, where $\*w_{t[D]}^\ast=1_{D\subset C_t}$;
this translates our online learning problem from the domain of measures on $C$ to the simplex on $|\D_T|$ elements.
However, we cannot play in this domain explicitly as we do not have access to the final discretization $\D_T$, nor do we get access to $\*w_t^\ast$ after task $t$, except implicitly via $C_t$.
In this section we design a method that implicitly run an online convex optimization procedure over $\R_{\ge0}^{|\D_T|}$ while explicitly playing probability measures $w:C\mapsto\R_{\ge0}$.

\begin{algorithm}[!t]
	\caption{Follow-the-Regularized-Leader (prescient form)}
	\label{alg:ftrl}
	\begin{algorithmic}[1]
		\STATE {\bfseries Input:} discretization $\D_T$ of $C$, mixture parameter $\gamma\in[0,1]$, step-size $\eta>0$
		\STATE Initialize $\*w_1=\*{\hat v}$
		\FOR{$t=1,2,\dots,T$}
		\STATE Play $\*w_t$.
		\STATE Suffer $f_t(\*w_t)=-\log\langle\*w_t^\ast,\*w_t\rangle$.
		\STATE Observe $f_t$.
		\STATE Update $\*w_{t+1}=\argmin_{\|\*w\|_1=1,\*w\ge\gamma\*{\hat v}}D_{KL}(\*w||\*{\hat v})+\eta\sum_{s\le t}f_s(\*w)$
		\ENDFOR
	\end{algorithmic}
\end{algorithm}

As the functions $f_t$ are exp-concave, one might first consider applying a method attaining logarithmic regret on such losses \cite{hazan2007logarithmic,orabona2012beyond};
however, such algorithms have regret that depends linearly on the dimension, which in our case is poly$(T)$.
We thus turn to the the follow-the-regularized-leader (FTRL) family of algorithms, which in the case of entropic regularization are well-known to have regret logarithmic in the dimension \cite{shalev-shwartz2011oco}.
In Algorithm~\ref{alg:ftrl} we display the pseudo-code of a modification with regularizer $D_{KL}(\cdot||\*{\hat v})$, where recall $\*v$ is the vector of volumes of the discretization $\D_T$ of $C$, and we constrain the played distribution to have measure at least $\gamma\*{\hat v}_{[D]}$ over every set $D\in\D_T$.

While Algorithm~\ref{alg:ftrl} explicitly requires knowing the discretization $\D_T$ of $C$ in advance, the following key lemma shows that we can run the procedure knowing only the discretization $\D_t$ after task $t$ by simply minimizing the same objective over probability distributions discretized on $\D_t$.
This crucially depends on the re-scaling of the entropic regularizer by $\*{\hat v}$ (which notably corresponds to the uniform distribution over $C$) and the fact that $\*w_t^\ast\in\{0,1\}^{|\D_T|}$.
\begin{Lem}\label{lem:equivalent}
	Let $w:C\mapsto\R_{\ge0}$ be the probability measure corresponding to the minimizer 
	\begin{equation}
	\*w=\argmin_{\|\*q\|_1=1,\*q\ge\gamma\*{\hat v}}D_{KL}(\*q||\*{\hat v})-\eta\sum_{s\le t}\log\langle\*w_s^\ast,\*q\rangle
	\end{equation}
	and let $\tilde w:C\mapsto\R_{\ge0}$ be the probability measure corresponding to the minimizer
	\begin{equation}
	\tilde{\*w}(t)=\argmin_{\|\*q\|_1=1,\*q\ge\gamma\*{\hat v}(t)}D_{KL}(\*q||\*{\hat v(t)})-\eta\sum_{s\le t}\log\langle\*w_s^\ast(t),\*q\rangle
	\end{equation}
	Then $\*w=\tilde{\*w}$.
\end{Lem}

We can thus move on to proving a regret guarantee for Algorithm~\ref{alg:ftrl}.
This follows from Jensen's inequality together with standard results for FTRL once we show that the loss functions are $\frac1{\gamma\vol(C_t)}$-Lipschitz over the constrained domain, yielding the following guarantee for Algorithm~\ref{alg:ftrl}:
\newpage
\begin{Thm}\label{thm:frl}
	Algorithm~\ref{alg:ftrl} has regret bounded by
	\begin{equation}
	\frac{1-\gamma}\eta D_{KL}(\*w^\ast||\*{\hat v})+\frac\eta{\gamma^2}\sum_{t=1}^T\frac1{(\vol(C_t))^2}+\gamma\sum_{t=1}^T\log\frac1{\vol(C_t)}
	\end{equation}
	w.r.t. the optimum in hindsight $\*w^\ast\in\argmin_{\|\*w\|_1=1,\*w\ge\*0}\sum_{t=1}^Tf_t(\*w)$ of the functions $f_t$.
	Setting $\gamma^2=GB/\sqrt T$ and $\eta^2=\frac{B^2\gamma^2}{TG^2}$, where $B^2=D_{KL}(\*w^\ast||\*{\hat v})$ and $G^2=\frac1T\sum_{t=1}^T\frac1{(\vol(C_t))^2}$, yields sublinear regret $\tilde O(\sqrt {BG}T^\frac34)$.
\end{Thm}

\begin{proof}
	Algorithm~\ref{alg:ftrl} is standard FTRL with regularizer $\frac1\eta D_{KL}(\cdot||\*{\hat v})$, which has the same Hessian as the standard entropic regularizer over the simplex and is thus $\frac1\eta$-strongly-convex w.r.t. $\|\cdot\|_1$ \cite[Example~2.5]{shalev-shwartz2011oco}.
	Applying Jensen's inequality, the standard regret bound for FTRL \cite[Theorem~2.11]{shalev-shwartz2011oco} together with the Lipschitz guarantee of Claim~\ref{clm:overlip}, and Jensen's inequality again yields the result:
	\begin{align*}
	\sum_{t=1}^Tf_t(\*w_t)-f_t(\*w^\ast)
	&=\sum_{t=1}^Tf_t(\*w_t)-(1-\gamma)f_t(\*w^\ast)-\gamma f_t(\*{\hat v})+\gamma(f_t(\*{\hat v})-f_t(\*w^\ast))\\
	&\le\sum_{t=1}^Tf_t(\*w_t)-f_t(\gamma\*{\hat v}+(1-\gamma)\*w^\ast)+\gamma\log\frac{\langle\*w_t^\ast,\*w^\ast\rangle}{\langle\*w_t^\ast,\*{\hat v}\rangle}\\
	&\le\frac1\eta D_{KL}(\gamma\*{\hat v}+(1-\gamma)\*w^\ast||\*{\hat v})+\frac\eta{\gamma^2}\sum_{t=1}^T\frac1{(\vol(C_t))^2}+\gamma\sum_{t=1}^T\log\frac1{\vol(C_t)}\\
	&\le\frac{1-\gamma}\eta D_{KL}(\*w^\ast||\*{\hat v})+\frac\eta{\gamma^2}\sum_{t=1}^T\frac1{(\vol(C_t))^2}+\gamma\sum_{t=1}^T\log\frac1{\vol(C_t)}
	\end{align*}
\end{proof}

Since the regret is sublinear in $T$, this result satisfies our requirement for attaining asymptotic improvement over single-task learning via Theorem~\ref{lem:aruba}.
However, there are several aspects of this bound that warrant some discussion.
The first is the rate of $T^\frac34$, which is less sublinear than the standard $\sqrt T$ and certainly the $\log T $ regret of exp-concave functions.
However, the functions we face are (a) non-Lipschitz and (b) over a domain that has dimensionality $\Omega(T)$;
both violate conditions for good rates in online convex optimization~\cite{hazan2007logarithmic,shalev-shwartz2011oco}, making our problem much more difficult.

A more salient aspect is the dependence on $B^2=D_{KL}(\*w^\ast||\*{\hat v})$, effectively the negative entropy of the optimal initialization.
This quantity is in-principle unbounded but is analogous to standard online convex optimization bounds that depend on the norm of the optimum, which in e.g. the Euclidean case are also unbounded.
In our case, if the optimal distribution is highly concentrated on a very small subset of the space it will be difficult to compete with.
Note that our setting of $\eta$ depends on knowing or guessing $B$;
this is also standard but is certainly a target for future work to address.
For example, past work on parameter-free algorithms has solutions for optimization over the simplex~\cite{orabona2016parameter};
however, it is unclear whether this is straightforward to do while preserving the property given by Lemma~\ref{lem:equivalent} allowing us to implicitly work with an unknown discretization.
A more reasonable approach may be to compete only with smooth measures that only assign probability at most $\kappa\vol(D)$ to any subset $D\subset C$ for some constant $\kappa\ge1$;
in this case we will simply have $B$ bounded by $\log\kappa$. 

A final issue is the dependence on $\sqrt G$, which is bounded by the reciprocal of the smallest volume $\vol(C_t)$, which in the dispersed case is roughly $O(m^{\beta d})$;
this means that the task-averaged regret will have a term that, while decreasing as we see additional tasks, is {\em increasing} in the number of within-task iterations and the dispersion parameter, which is counter-intuitive.
It is also does so exponentially in the dimension.
Note that in the common algorithm configuration setting of $\beta=1/2$ and $d=1$ this will simply mean that for each task we suffer an extra $o_T(1)$ loss at each within-task round, a quantity which vanishes asymptotically.

\subsection{Meta-learning the step-size}

In addition to learning the initialization, Theorem~\ref{lem:aruba} requires learning the task-similarity to set the within-task step-size $\lambda>0$.
This involves optimizing functions of form $h_t(v)=v+f_t(w_t)/v$.
Since we know that the measures $w_t$ are lower-bounded in terms of $\gamma$, we can apply a previous result \cite{khodak2019adaptive} that solves this by running the EWOO algorithm \cite{hazan2007logarithmic} on the modified sequence $v+\frac{f_t(w_t)+\varepsilon^2}v$:
\begin{Cor}\label{cor:ewoo}
	For any $\varepsilon>0$, running the EWOO algorithm on the modified sequence $v+\frac{f_t(w)+\varepsilon^2}v$ over the domain $[\varepsilon,\sqrt{D^2-\log\gamma+\varepsilon^2}]$, where $D^2\ge\frac1T\sum_{t=1}^T\log\frac1{\vol(C_t)}$, attains regret
	\begin{equation}
	\min\left\{\frac{\varepsilon^2}{v^\ast},\varepsilon\right\}T+\frac {\sqrt{D^2-\log\gamma}}2\max\left\{\frac{D^2-\log\gamma}{\varepsilon^2},1\right\}(1+\log(T+1))
	\end{equation}
	on the original sequence $h_t(v)=v+f_t(w)/v$ for all $v^\ast>0$.
\end{Cor}

Setting $\varepsilon=1/\sqrt[4]T$ gives a guarantee of form $\tilde O((\min\{1/v^\ast,\sqrt[4]T\})\sqrt T)$.
Note this rate might be improvable by using the fact that $v$ is lower-bounded due to the $\gamma$-constraint; 
however, we do not focus on this since this component is not the dominant term in the regret.
In fact, because of this we can adapt a related method that simply runs follow-the-leader (FTL) on the same modified sequence~\cite{khodak2019adaptive} without affecting the dominant terms in the regret:
\begin{Cor}\label{cor:ftl}
	For any $\varepsilon>0$, running the FTL algorithm on the modified sequence $v+\frac{f_t(w)+\varepsilon^2}v$ over the domain $[\varepsilon,\sqrt{D^2-\log\gamma+\varepsilon^2}]$, where $D^2\ge\frac1T\sum_{t=1}^T\log\frac1{\vol(C_t)}$, attains regret
	\begin{equation}
	\min\left\{\frac{\varepsilon^2}{v^\ast},\varepsilon\right\}T+2\sqrt{D^2-\log\gamma}\max\left\{\frac{(D^2-\log\gamma)^\frac32}{\varepsilon^3},1\right\}(1+\log(T+1))
	\end{equation}
	on the original sequence $h_t(v)=v+f_t(w)/v$ for all $v^\ast>0$.
\end{Cor}
Setting $\varepsilon=1/\sqrt[5]T$ gives a guarantee of form $\tilde O((\min\{1/v^\ast,\sqrt[5]T\})T^\frac35)$.
The alternatives are described in pseudocode at the bottom of Algorithm~\ref{alg:meta};
while the guarantee of the FTL-based approach is worse, it is almost as simple to compute as the task-similarity and does not require integration, making it easier to implement.

\subsection{Putting the two together}

\begin{algorithm}[!t]
	\caption{
		Meta-learning the parameters of the exponential forecaster (Algorithm~\ref{alg:ef}).
		Recall that $\*p(t)$ refers to the time-$t$ discretization of the measure $p:C\mapsto\R_{\ge0}$ (c.f. Section~\ref{sec:meta}).
	}
	\label{alg:meta}
	\begin{algorithmic}[1]
		\STATE {\bfseries Input:} domain $C\subset\R^d$, dispersion $\beta>0$, step-size $\eta>0$, constraint parameter $\gamma\in[0,1]$, offset parameter $\varepsilon>0$, domain parameter $D>0$.
		\STATE Initialize $w_1$ to the uniform measure on $C$ and set $\lambda_1=\frac{\varepsilon+\sqrt{D^2+\varepsilon^2-\log\gamma}}{2\sqrt m}$.
		\FOR{task $t=1,2,\dots,T$}
		\STATE Run Algorithm~\ref{alg:ef} with initialization $w_t$ and step-size $\lambda_t$ and obtain task-$t$ optimum $\rho_t^\ast\in C$.
		\STATE Set $w_t^\ast=1_{\B(\rho_t^\ast,m^{-\beta})}$ to be the function that is 1 in the $m^{-\beta}$-ball round $\rho_t^\ast$ and 0 elsewhere.
		\STATE Set $w_{t+1}$ to $\*w_{t+1}(t)=\argmin_{\|\*w\|_1=1,\*w\ge\gamma\*{\hat v}(t)}D_{KL}(\*w||\*{\hat v}(t))-\eta\sum_{s\le t}\log\langle\*w_s^\ast(t),\*w\rangle$.
		\IF{using EWOO}
			\STATE Define $\mu_t(x)=\exp\left(-\alpha\left(tx+\frac{t\varepsilon^2-\sum_{s\le t}\log\langle\*w_s^\ast(s),\*w_s(s)\rangle}x\right)\right)$ for $\alpha=\frac2D\min\left\{\frac{\varepsilon^2}{D^2},1\right\}$.
			\STATE Set $\lambda_{t+1}=\frac{\int_\varepsilon^{\sqrt{D^2+\varepsilon^2-\log\gamma}}x\mu_t(x)dx}{\sqrt m\int_\varepsilon^{\sqrt{D^2+\varepsilon^2-\log\gamma}}\mu_t(x)dx}$.
		\ELSE
			\STATE Set $\lambda_{t+1}=\sqrt{\frac{\sum_{s\le t}\varepsilon^2-\log\langle\*w_s^\ast(s),\*w_s(s)\rangle}{tm}}$.
		\ENDIF
		\ENDFOR
	\end{algorithmic}
\end{algorithm}

Now that we have an algorithm for both the initialization and the step-size, we can combine the two in Algorithm~\ref{alg:meta} to meta-learn the parameter of the exponential forecaster.
Then we can obtain a bound on the task-averaged regret from Theorem~\ref{lem:aruba} to attain our final result.

\begin{Thm}\label{thm:tar}
	Define $B^2=D_{KL}(\*w^\ast||\*{\hat v})$, $G^2=\frac1T\sum_{t=1}^T\frac1{(\vol(C_t))^2}$, and $D^2\ge\frac1T\sum_{t=1}^T\log\frac1{\vol(C_t)}=O(\beta d\log m)$.
	Then Algorithm~\ref{alg:meta} with $\eta,\gamma$ set as in Theorem~\ref{thm:frl} and $\varepsilon=1/\sqrt[4]T$ (if using EWOO) or $1/\sqrt[5]T$ (otherwise) yields task-averaged regret
	\begin{equation}
	\tilde O\left(\min\left\{\frac{\sqrt{BG}}{V\sqrt[4]T},\frac{\sqrt[4]{BG}}{\sqrt[8]T}\right\}+2V\right)\sqrt m+g(m)
	\end{equation}
	Here $V$ is the task-similarity \eqref{eq:tasksim}.
\end{Thm}

So as in past work in meta-learning, this achieves the goal of adapting to the task-similarity by attaining asymptotic regret of $2V\sqrt m+O(m^{-\beta})$ on-average, where here we substitute the dispersion term for $g$ and $V^2$ is the task-similarity encoding the average probability mass assigned to the different task balls by the optimal initialization distribution.
We include the minimum of two rates in the bound, with the rate being $1/\sqrt[4]T$ is the task-similarity is a constant $\Theta_T(1)$ and $1/\sqrt[8]T$ if it is extremely small.
As discussed in above, this rate reflects the difficulty of our meta-problem, in which we are optimizing non-smooth functions over a space of distributions;
in contrast, past meta-update procedures have taken advantage of nice properties of Bregman divergences to obtain faster rates \cite{khodak2019adaptive}.

\section{Meta-learning for data-driven algorithm design}

We demonstrate the utility of our bounds in a series of applications across two general areas: 
data-driven algorithm design \cite{balcan2020data} and robust learning.
This section focuses on the former and demonstrates how our results imply guarantees for meta-learning the tuning of solvers for several difficult combinatorial problems arising from the theory of computing.
We also demonstrate the practical utility of our approach for tuning clustering algorithms on real and synthetic datasets.

\subsection{Instantiations for tuning combinatorial optimization algorithms}
Algorithm configuration for combinatorial optimization algorithms involves learning algorithm parameters from multiple instances of combinatorial problems \cite{gupta2017pac,balcan2017learning,balcan2020data}. For well-known problems like MWIS (maximum weighted independent set), IQP (integer quadratic programming), and mechanism design for auctions, the algorithmic performance on a fixed instance is typically a piecewise Lipschitz function of the algorithm parameters. Prior work has looked at learning these parameters in the distributional setting (i.e. assuming iid draws of problem instances) \cite{balcan2017learning} or the online setting where the problem instances may be adversarially drawn \cite{balcan2018dispersion,balcan2020learning}. On the other hand, instantiating our results for these problems provide upper bounds for much more realistic settings where different tasks may be related and our bounds improve with this relatedness.

We demonstrate how to apply our results to several combinatorial problems under mild smoothness assumptions. The key idea is to show that if the inputs come from a smooth distribution, the algorithmic performance is dispersed (as a sequence of functions in the algorithm parameters). 
We leverage known results about the MWIS problem to show $\frac{1}{2}$-dispersion, which together with Theorem \ref{thm:tar} implies that our bound on the task-averaged regret improves with task similarity $V$. 

{\bf The MWIS problem.} In MWIS, there is a graph $G=(V,E)$
and a weight $w_v\in\R^+$ for each vertex $v\in V$. The goal is to find a set of non-adjacent vertices
with maximum total weight. The problem is $NP$-hard and in fact does not have any constant factor polynomial time approximation algorithm. \cite{gupta2017pac} propose a greedy heuristic family, which selects vertices greedily based on largest value of $w_v / (1 + \text{deg}(v))^\rho$, where $\text{deg}(v)$ is the degree of vertex $v$, and removes neighbors of the selected vertex before selecting the next vertex.

For this algorithm family, we can learn the best parameter $\rho$ provided pairs of vertex weights have a joint $\kappa$-bounded distribution, and Theorem \ref{thm:tar} implies regret bounds that improve with task similarity. We use the recipe from \cite{balcan2020semi} to establish dispersion.

\begin{Thm}\label{thm:mwis-tar}
Consider instances of MWIS with all vertex weights in $(0, 1]$ and for each instance, every pair of vertex weights has a $\kappa$-bounded joint distribution. Then the asymptotic task-averaged regret for learning the algorithm parameter $\rho$ is $o_T(1)+2V\sqrt{m}+O(\sqrt{m})$.
\end{Thm}

\begin{proof}[Proof sketch]
The loss function is piecewise constant with discontinuities corresponding to $\rho$ such that $w_v / (1 + \text{deg}(v))^\rho=w_u / (1 + \text{deg}(u))^\rho$ for a pair of vertices $u,v$. \cite{balcan2018dispersion} show that the discontinuities have $(\kappa \ln n)$-bounded distributions where $n$ is the number of vertices. This implies that in any interval of length $\epsilon$, we have in expectation at most $\epsilon\kappa \ln n$ discontinuities. Using this in dispersion recipe from \cite{balcan2020semi} implies $\frac{1}{2}$-dispersion, which in turn implies the desired regret bound by applying Theorem \ref{thm:tar}.
\end{proof}

Similar results may be obtained for other combinatorial problems including knapsack, $k$-center clustering, IQP and auction design (see Appendix \ref{app: combinatorial} for full details). We further show instantiations of our results for knapsack and $k$-center clustering, for which we will empirically validate our proposed methods in the next sections.

{\bf Greedy Knapsack.} Knapsack is a well-known NP-complete problem. We are given a knapsack with capacity $\texttt{cap}$ and items $i\in[m]$ with sizes $w_i$ and values $v_i$. The goal is to select a subset $S$ of items to add to the knapsack such that $\sum_{i\in S}w_i\le \texttt{cap}$ while maximizing the total value $\sum_{i\in S}v_i$ of selected items. The classic greedy heuristic to add items in decreasing order of $v_i/w_i$ gives a 2-approximation. We consider a generalization to use $v_i/w_i^{\rho}$ proposed by \cite{gupta2017pac} for $\rho\in[0,10]$. For example, for the value-weight pairs $\{(0.99,1),(0.99,1),(1.01,1.01)\}$ and capacity $\texttt{cap}=2$ the classic heuristic $\rho=1$ gives value $1.01$ but using $\rho=3$ gives the optimal value $1.98$. We can learn this optimal value of $\rho$ from similar tasks, and obtain formal guarantees similar to Theorem \ref{thm:mwis-tar} (proof in Appendix \ref{app: combinatorial}).

\begin{Thm}
Consider instances of the knapsack problem given by bounded weights $w_{i,j}\in[1,C]$ and $\kappa$-bounded independent values $v_{i,j}\in[0,1]$ for $i\in[m],j\in[T]$. Then the asymptotic task-averaged regret for learning the algorithm parameter $\rho$ for the greedy heuristic family described above is $o_T(1)+2V\sqrt{m}+O(\sqrt{m})$.
\end{Thm}

{\bf $k$-center clustering.} We consider the parameterized $\alpha$-Llyod's algorithm family introduced in \cite{balcan2018data}. In the seeding phase, each point $x$ is sampled with probability proportional to $\min_{c\in C}d(v, c)^{\alpha}$, where $d(\cdot,\cdot)$ is the distance metric and $C$ is the set of centers chosen so far. The family contains an algorithm for each $\alpha\in[0,\infty)\cup \infty$, and includes popular clustering heuristics like vanilla $k$-means (random initial centers, for $\alpha=0$), $k$-means++ (corresponding to $\alpha=2$) and farthest-first traversal ($\alpha=\infty$). The performance of the algorithm is measured using the Hamming distance to the optimal clustering, and is a piecewise constant function of $\alpha$. Our meta-learning result can be instantiated for this problem even without smoothness assumptions (simply leveraging the smoothness induced by the internal randomness of the clustering algorithm, proof in Appendix \ref{app: combinatorial}).

\begin{Thm}
Consider instances of the $k$-center clustering problem on $n$ points, with Hamming loss $l_{i,j}$ for $i\in[m],j\in[T]$ against some (unknown) ground truth clustering. Then the asymptotic task-averaged regret for learning the algorithm parameter $\alpha$ for the $\alpha$-Lloyd's clustering algorithm family of \cite{balcan2018data} is $o_T(1)+2V\sqrt{m}+O(\sqrt{m})$.
\end{Thm}

In the following section we look at applications of our results through experiments for the knapsack and $k$-center clustering problems.

\subsection{Experiments for greedy knapsack and $k$-center clustering}\label{sec:experiments} We design experiments to evaluate our new meta-initialization algorithm for data-driven design for knapsack and clustering problems on real and simulated data. Our experiments show the usefulness of our techniques in learning a sequence of piecewise-Lipschitz functions.

For our experiments, we generate a synthetic dataset of knapsack instances described as follows. For each problem instance of each task, we have $\texttt{cap}=100$ and $m=50$. We have $10$ `heavy' items with $w_i\sim \mathcal{N}(27,0.5)$ and $v_i\sim \mathcal{N}(27,0.5)$, and $40$ items with $w_i\sim \mathcal{N}(19+w_t,0.5)$ and $v_i\sim \mathcal{N}(18,0.5)$, where $w_t\in[0,2]$ is task-dependent.

We also consider the parameterized $\alpha$-Lloyd's algorithm family introduced in \cite{balcan2018data}. The performance of the algorithm is measured using the Hamming loss relative to the optimal clustering, and is a piecewise constant function of $\alpha$. We can compute the pieces of this function for $\alpha\in[0,10]$ by iteratively computing the subset of parameter values where a candidate point can be the next center. We use the small split of the {\it Omniglot} dataset \cite{lake2015human}, and create clustering tasks by drawing random samples consisting of five characters each, where four characters are constant throughout. We also create a Gaussian mixture binary classification dataset where each class is a 2D Gaussian distribution consisting of 100 points each, with variance $\begin{pmatrix}
  \sigma & 0\\ 
  0 & 2\sigma
\end{pmatrix}$ and centers $(0,0)$ and $(d\sigma,0)$. We pick $d\in[2,3]$ to create different tasks.

For each dataset we learn using 30 instances each of 10 training tasks and evaluate average loss over 5 test tasks. We perform 100 iterations to average over the randomization of the clustering algorithm and the exponential forecaster algorithm.
We perform meta-initialization with parameters $\gamma=\eta=0.01$ (no hyperparameter search performed). The step-size is set to minimize the regret term in Theorem \ref{thm:exp-forc-meta}, and not meta-learned.

The relative improvement in task-averaged regret due to meta-learning in our formal guarantees depend on the task-similarity $V$ and how it compares to the dispersion-related $O(m^{1-\beta})$ term, and can be significant when the latter is small. Our results in Table~\ref{table: meta initialization} show that meta-learning an initialization, i.e. a distribution over the algorithm parameter, for the exponential forecaster in this setting yields improved performance on each dataset. We observe this for both the one-shot and five-shot settings, i.e. the number of within-task iterations of the test task are one and five respectively.
The benefit of meta-learning is most pronounced for the Gaussian mixture case (well-dispersed and similar tasks), and gains for Omniglot may increase with more tasks (dispersed but less similar tasks). For our knapsack dataset, the relative gains are smaller (similar tasks, but less dispersed). See Appendix \ref{app: experiment} for further experiments that lead us to these insights.

\begin{table*}[t]
	\centering
	\caption{Effect of meta-initialization on few-shot learning of algorithmic parameters. Performance is computed as a fraction of the average value (Hamming accuracy, or knapsack value) of the offline optimum parameter.}
	\label{table: meta initialization}
	\resizebox{0.98\textwidth}{!}{%
\begin{tabular}{c||cc|cc|cc}
\toprule
Dataset & \multicolumn{2}{c}{Omniglot} & \multicolumn{2}{c}{Gaussian Mixture} & \multicolumn{2}{c}{Knapsack} \\
& One-shot & Five-shot & One-shot & Five-shot& One-shot & Five-shot \\
\midrule
\midrule
Single task & $88.67\pm0.47\%$ & $95.02\pm0.19\%$ & $90.10\pm1.10\%$ & $91.43\pm0.44\%$ &$84.74\pm0.29\%$&$98.89\pm0.17\%$\\
Meta-initialized & $89.65\pm0.49\%$ & $96.05\pm0.15\%$ & $95.76\pm0.60\%$ & $96.39\pm0.27\%$&$85.66\pm0.57\%$&$99.12\pm0.15\%$
\\
\bottomrule
\end{tabular}
}
\end{table*}

\section{Robust online meta-learning}

In online learning, we seek to minimize a sequence of loss functions, and are required to perform well relative to the optimal choice in hindsight. It is possible for the observed loss functions to be noisy on some inputs, either naturally or due to adversarial intent. We will now explore the conditions under which learning robust to such an adversarial influence (i.e. outlier injection) is possible, which is particularly common in meta-learning with diverse sources.

{\it Setup}: At round $i$, we play $x_i$, observe perturbed loss $\Tilde{l}_i : \X\rightarrow[0,1]$ which is set by the adversary by modifying the true loss $l_i:\X\rightarrow[0,1]$ using an {\it attack function} $a_i:\X\rightarrow[0,1]$ such that $\Tilde{l}_i=l_i+a_i$ and may be non-Lipschitz, and suffer perturbed loss $\Tilde{l}_i(x_i)$ and true loss $l_i(x_i)$. We seek to minimize regret relative to best fixed action in hindsight, i.e. $$\Tilde{R}_m=\sum_{i=1}^m \Tilde{l}_i(x_i) - \min_{x\in\X}\sum_{i=1}^m \Tilde{l}_i(x)$$ for the perturbed loss and regret $$R_m=\sum_{i=1}^m l_i(x_i) - \min_{x\in\X}\sum_{i=1}^m l_i(x)$$ for the true loss.

No regret can be achieved provided the adversary distribution is sufficiently smooth, i.e. satisfies $\beta$-dispersion for some $\beta>0$, as this corresponds to online optimization of the perturbed loss function. We can show this for both perturbed and true loss. The perturbed loss guarantee is immediate from standard results on online learning of piecewise Lipschitz functions \cite{balcan2018dispersion,balcan2020learning}. For the true loss, we can achieve no regret if the adversary perturbation $a_i$ is limited to small balls and the centers of the balls are dispersed, which we capture using the following definition.
\begin{Def}[{$\delta$-bounded, $\beta_a$-dispersed attack}] An attack function $a_i$ is $\delta$-bounded if there exists a ball $\B(x_a,\delta)$ of radius $\delta$ such that $a_i(x)=0$ for each $x\in\X\setminus \B(x_a,\delta)$. $x_a$ is called a {\it center} $c_{a_i}$ for attack $a_i$. A sequence of attack functions $a_1,\dots,a_m$ is said to be $\beta_a$-dispersed, if the positions of attack centers $x_a$ are dispersed i.e. for all $m$ and for all $\epsilon\ge m^{-\beta_a}$,
$$\E\left[
\max_{x,x'\in\X,x\in\B(x',\epsilon)}\big\lvert
\{ i\in[m] \mid x=c_{a_i}\} \big\rvert \right] 
\le  \Tilde{O}(\epsilon m)$$.

\end{Def}

\begin{Thm}\label{thm:robustness single task}
Given a sequence of $\beta$-dispersed adversarially perturbed losses $\Tilde{l}_i=l_i+a_i$, where $\Tilde{l}_i,l_i,a_i$ are piecewise $L$-Lipschitz functions $ \X\rightarrow[0,1]$ for $i=1,\dots,m$ and $\X\subset\R^d$, the exponential forecaster algorithm has 
$$\E[\Tilde{R}_m]=\Tilde{O}(m\lambda +\frac{\log (1/Z)}{\lambda}+(L+1)m^{1-\beta})$$ (with Z as in Theorem \ref{thm:exp-forc-meta}). If in addition we have that $a_i$ is a $m^{-\beta_a}$-bounded, $\beta_a$-dispersed attack, then 
     $$\E[R_m]=\Tilde{O}(m\lambda +\frac{\log (1/Z)}{\lambda}+(L+1)m^{1-\min\{\beta,\beta_a\}}).$$
\end{Thm}

Together with Theorem \ref{thm:tar}, this implies no regret meta-learning in the presence of dispersed adversaries, in particular the occurrence of unreliable data in small dispersed parts of the domain. We also show a lower bound below which establishes that our upper bounds are essentially optimal in the attack dispersion.

\begin{Thm}\label{thm:robustness lower bound}
There exist sequences of piecewise $L$-Lipschitz functions $\Tilde{l}_i,l_i,a_i$ $[0,1]\rightarrow[0,1]$ for $i=1,\dots,m$ such that for any online algorithm \begin{enumerate}\itemsep0em
    \item $\Tilde{l}_i$ is $\beta$-dispersed and $\E[\Tilde{R}_m]=\Omega(m^{1-\beta})$,
    \item $\Tilde{l}_i$ is $\beta$-dispersed, $a_i$ is $m^{-\beta}$-bounded, $\beta_a$-dispersed and $\E[R_m]=\Omega(m^{1-\min\{\beta,\beta_a\}})$.
\end{enumerate}  
\end{Thm}

\section{Conclusion}

In this paper we studied the initialization-based meta-learning of piecewise-Lipschitz functions, demonstrating how online convex optimization over an adaptive discretization can find an initialization that improves the performance of the exponential forecaster across tasks, assuming the tasks have related optima. 
We then applied this result in two settings:
online configuration of clustering algorithms and adversarial robustness in online learning.
For the latter we introduced a dispersion-based understanding of robustness that we believe to be of independent interest.
In addition, there are further interesting applications of our work to other algorithm configuration problems.

\section*{Acknowledgments}

This material is based on work supported in part by the National Science Foundation under grants CCF-1535967, CCF-1910321, IIS-1618714, IIS-1705121, IIS-1838017, IIS-1901403, IIS-2046613, and SES-1919453; the Defense Advanced Research Projects Agency under cooperative agreements HR00112020003 and FA875017C0141; an AWS Machine Learning Research Award; an Amazon Research Award; a Bloomberg Research Grant; a Microsoft Research Faculty Fellowship; an Amazon Web Services Award; a Facebook Faculty Research Award; funding from Booz Allen Hamilton Inc.; and a Block Center Grant.
Any opinions, findings and conclusions or recommendations expressed in this material are those of the author(s) and do not necessarily reflect the views of any of these funding agencies.

\bibliography{refs}

\begin{thebibliography}{10}

\bibitem{agarwal2019online}
Naman Agarwal, Brian Bullins, Elad Hazan, Sham Kakade, and Karan Singh.
\newblock Online control with adversarial disturbances.
\newblock In {\em International Conference on Machine Learning}, pages
  111--119. PMLR, 2019.

\bibitem{alquier2017lifelong}
Pierre Alquier, The~Tien Mai, and Massimiliano Pontil.
\newblock Regret bounds for lifelong learning.
\newblock In {\em Proceedings of the 20th International Conference on
  Artificial Intelligence and Statistics}, 2017.

\bibitem{bagdasaryan2020backdoor}
Eugene Bagdasaryan, Andreas Veit, Yiqing Hua, Deborah Estrin, and Vitaly
  Shmatikov.
\newblock How to backdoor federated learning.
\newblock In {\em International Conference on Artificial Intelligence and
  Statistics}, pages 2938--2948. PMLR, 2020.

\bibitem{balcan2020data}
Maria-Florina Balcan.
\newblock Book chapter {Data-Driven Algorithm Design}.
\newblock In {\em Beyond Worst Case Analysis of Algorithms, T. Roughgarden
  (Ed)}. Cambridge University Press, 2020.

\bibitem{balcan2020power}
Maria-Florina Balcan, Avrim Blum, Dravyansh Sharma, and Hongyang Zhang.
\newblock On the power of abstention and data-driven decision making for
  adversarial robustness.
\newblock {\em arXiv preprint arXiv:2010.06154}, 2020.

\bibitem{balcan2015lifelong}
Maria-Florina Balcan, Avrim Blum, and Santosh Vempala.
\newblock Efficient representations for lifelong learning and autoencoding.
\newblock In {\em Proceedings of the 28th Annual Conference on Learning
  Theory}, 2015.

\bibitem{balcan2019learning}
Maria-Florina Balcan, Travis Dick, and Manuel Lang.
\newblock Learning to link.
\newblock In {\em International Conference on Learning Representations}, 2019.

\bibitem{balcan2020semi}
Maria-Florina Balcan, Travis Dick, and Wesley Pegden.
\newblock Semi-bandit optimization in the dispersed setting.
\newblock In {\em Conference on Uncertainty in Artificial Intelligence}, pages
  909--918. PMLR, 2020.

\bibitem{balcan2018learning}
Maria-Florina Balcan, Travis Dick, Tuomas Sandholm, and Ellen Vitercik.
\newblock Learning to branch.
\newblock In {\em International conference on machine learning}, pages
  344--353. PMLR, 2018.

\bibitem{balcan2020learning}
Maria-Florina Balcan, Travis Dick, and Dravyansh Sharma.
\newblock {Learning piecewise Lipschitz functions in changing environments}.
\newblock In {\em Proceedings of the 23rd International Conference on
  Artificial Intelligence and Statistics}, pages 3567--3577, 2020.

\bibitem{balcan2018dispersion}
Maria-Florina Balcan, Travis Dick, and Ellen Vitercik.
\newblock Dispersion for data-driven algorithm design, online learning, and
  private optimization.
\newblock In {\em 2018 IEEE 59th Annual Symposium on Foundations of Computer
  Science (FOCS)}, pages 603--614, 2018.

\bibitem{balcan2017learning}
Maria-Florina Balcan, Vaishnavh Nagarajan, Ellen Vitercik, and Colin White.
\newblock Learning-theoretic foundations of algorithm configuration for
  combinatorial partitioning problems.
\newblock In {\em Annual Conference on Learning Theory}, pages 213--274, 2017.

\bibitem{balcan2018general}
Maria-Florina Balcan, Tuomas Sandholm, and Ellen Vitercik.
\newblock A general theory of sample complexity for multi-item profit
  maximization.
\newblock In {\em Proceedings of the 2018 ACM Conference on Economics and
  Computation}, pages 173--174, 2018.

\bibitem{balcan2021data}
Maria-Florina Balcan and Dravyansh Sharma.
\newblock Data driven algorithms for limited labeled data learning.
\newblock {\em arXiv preprint arXiv:2103.10547}, 2021.

\bibitem{balcan2018data}
Maria-Florina~F Balcan, Travis Dick, and Colin White.
\newblock Data-driven clustering via parameterized lloyd's families.
\newblock {\em Advances in Neural Information Processing Systems},
  31:10641--10651, 2018.

\bibitem{blum2020technical}
Avrim Blum.
\newblock Technical perspective: Algorithm selection as a learning problem.
\newblock {\em Communications of the ACM}, 63(6):86--86, 2020.

\bibitem{brendel2020adversarial}
Wieland Brendel, Jonas Rauber, Alexey Kurakin, Nicolas Papernot, Behar Veliqi,
  Sharada~P Mohanty, Florian Laurent, Marcel Salath{\'e}, Matthias Bethge,
  Yaodong Yu, et~al.
\newblock Adversarial vision challenge.
\newblock In {\em The NeurIPS'18 Competition}, pages 129--153. Springer, 2020.

\bibitem{chen2018fedmeta}
Fei Chen, Zhenhua Dong, Zhenguo Li, and Xiuqiang He.
\newblock Federated meta-learning for recommendation.
\newblock arXiv, 2018.

\bibitem{cohen2017online}
Vincent Cohen-Addad and Varun Kanade.
\newblock Online optimization of smoothed piecewise constant functions.
\newblock In {\em Proceedings of the 20th International Conference on
  Artificial Intelligence and Statistics}, 2017.

\bibitem{denevi2019ltlsgd}
Giulia Denevi, Carlo Ciliberto, Riccardo Grazzi, and Massimiliano Pontil.
\newblock Learning-to-learn stochastic gradient descent with biased
  regularization.
\newblock In {\em Proceedings of the 36th International Conference on Machine
  Learning}, 2019.

\bibitem{denevi2019meta}
Giulia Denevi, Carlo Ciliberto, Riccardo Grazzi, and Massimiliano Pontil.
\newblock Online-within-online meta-learning.
\newblock In {\em Advances in Neural Information Processing Systems}, 2019.

\bibitem{du2021fewshot}
Simon~S. Du, Wei Hu, Sham~M. Kakade, Jason~D. Lee, and Qi~Lei.
\newblock Few-shot learning via learning the representation, provably.
\newblock In {\em Proceedings of the 9th International Conference on Learning
  Representations}, 2021.

\bibitem{duan2017imitation}
Yan Duan, Marcin Andrychowicz, Bradly Stadie, Jonathan Ho, Jonas Schneider,
  Ilya Sutskever, Pieter Abbeel, and Wojciech Zaremba.
\newblock One-shot imitation learning.
\newblock In {\em Advances in Neural Information Processing Systems}, 2017.

\bibitem{fallah2020meta}
Alireza Fallah, Aryan Mokhtari, and Asuman Ozdaglar.
\newblock On the convergence theory of gradient-based model-agnostic
  meta-learning algorithms.
\newblock In {\em Proceedings of the 23rd International Conference on
  Artificial Intelligence and Statistics}, 2020.

\bibitem{finn2017maml}
Chelsea Finn, Pieter Abbeel, and Sergey Levine.
\newblock Model-agnostic meta-learning for fast adaptation of deep networks.
\newblock In {\em Proceedings of the 34th International Conference on Machine
  Learning}, 2017.

\bibitem{goemans1995improved}
Michel~X Goemans and David~P Williamson.
\newblock Improved approximation algorithms for maximum cut and satisfiability
  problems using semidefinite programming.
\newblock {\em Journal of the ACM (JACM)}, 42(6):1115--1145, 1995.

\bibitem{gupta2017pac}
Rishi Gupta and Tim Roughgarden.
\newblock A {PAC} approach to application-specific algorithm selection.
\newblock {\em SIAM Journal on Computing}, 46(3):992--1017, 2017.

\bibitem{hazan2007logarithmic}
Elad Hazan, Amit Agarwal, and Satyen Kale.
\newblock Logarithmic regret algorithms for online convex optimization.
\newblock {\em Machine Learning}, 69:169--192, 2007.

\bibitem{khodak2019adaptive}
Mikhail Khodak, Maria-Florina Balcan, and Ameet Talwalkar.
\newblock Adaptive gradient-based meta-learning methods.
\newblock In {\em Advances in Neural Information Processing Systems}, 2019.

\bibitem{khodak2019provable}
Mikhail Khodak, Maria-Florina Balcan, and Ameet Talwalkar.
\newblock Provable guarantees for gradient-based meta-learning.
\newblock In {\em Proceedings of the 36th International Conference on Machine
  Learning}, 2019.

\bibitem{khodak2021fedex}
Mikhail Khodak, Renbo Tu, Tian Li, Liam Li, Maria-Florina Balcan, Virginia
  Smith, and Ameet Talwalkar.
\newblock Federated hyperparameter tuning: Challenges, baselines, and
  connections to weight-sharing.
\newblock arXiv, 2021.

\bibitem{kong2020robust}
Weihao Kong, Raghav Somani, Sham Kakade, and Sewoong Oh.
\newblock Robust meta-learning for mixed linear regression with small batches.
\newblock {\em Advances in Neural Information Processing Systems}, 33, 2020.

\bibitem{lake2015human}
Brenden~M Lake, Ruslan Salakhutdinov, and Joshua~B Tenenbaum.
\newblock Human-level concept learning through probabilistic program induction.
\newblock {\em Science}, 350(6266):1332--1338, 2015.

\bibitem{li2020dp}
Jeffrey Li, Mikhail Khodak, Sebastian Caldas, and Ameet Talwalkar.
\newblock Differentially private meta-learning.
\newblock In {\em Proceedings of the 8th International Conference on Learning
  Representations}, 2020.

\bibitem{maurer2016mtl}
Andreas Maurer, Massimiliano Pontil, and Bernardino Romera-Paredes.
\newblock The benefit of multitask representation learning.
\newblock {\em Journal of Machine Learning Research}, 17(1):2853--2884, 2016.

\bibitem{nguyen2015deep}
Anh Nguyen, Jason Yosinski, and Jeff Clune.
\newblock Deep neural networks are easily fooled: High confidence predictions
  for unrecognizable images.
\newblock In {\em Proceedings of the IEEE conference on computer vision and
  pattern recognition}, pages 427--436, 2015.

\bibitem{orabona2012beyond}
Francesco Orabona, Nicolo Cesa-Bianchi, and Claudio Gentile.
\newblock Beyond logarithmic bounds in online learning.
\newblock In {\em Proceedings of the Fifteenth International Conference on
  Artificial Intelligence and Statistics}, 2012.

\bibitem{orabona2016parameter}
Francesco Orabona and David Pal.
\newblock Coin betting and parameter-free online learning.
\newblock In {\em Advances in Neural Information Processing Systems}, 2016.

\bibitem{pillutla2019robust}
Krishna Pillutla, Sham~M Kakade, and Zaid Harchaoui.
\newblock Robust aggregation for federated learning.
\newblock {\em arXiv preprint arXiv:1912.13445}, 2019.

\bibitem{raghu2020anil}
Aniruddh Raghu, Maithra Raghu, Samy Bengio, and Oriol Vinyals.
\newblock Rapid learning or feature reuse? {T}owards understanding the
  effectiveness of {MAML}.
\newblock In {\em Proceedings of the 8th International Conference on Learning
  Representations}, 2020.

\bibitem{resler2019adversarial}
Alon Resler and Yishay Mansour.
\newblock Adversarial online learning with noise.
\newblock In {\em International Conference on Machine Learning}, pages
  5429--5437. PMLR, 2019.

\bibitem{saunshi2020meta}
Nikunj Saunshi, Yi~Zhang, Mikhail Khodak, and Sanjeev Arora.
\newblock A sample complexity separation between non-convex and convex
  meta-learning.
\newblock In {\em Proceedings of the 37th International Conference on Machine
  Learning}, 2020.

\bibitem{shalev-shwartz2011oco}
Shai Shalev-Shwartz.
\newblock Online learning and online convex optimization.
\newblock {\em Foundations and Trends in Machine Learning}, 4(2):107--194,
  2011.

\bibitem{spielman2004smoothed}
Daniel~A Spielman and Shang-Hua Teng.
\newblock Smoothed analysis of algorithms: Why the simplex algorithm usually
  takes polynomial time.
\newblock {\em Journal of the ACM (JACM)}, 51(3):385--463, 2004.

\bibitem{thrun1998ltl}
Sebastian Thrun and Lorien Pratt.
\newblock {\em Learning to Learn}.
\newblock Springer Science \& Business Media, 1998.

\bibitem{tolpegin2020data}
Vale Tolpegin, Stacey Truex, Mehmet~Emre Gursoy, and Ling Liu.
\newblock Data poisoning attacks against federated learning systems.
\newblock In {\em European Symposium on Research in Computer Security}, pages
  480--501. Springer, 2020.

\bibitem{tripuraneni2021provable}
Nilesh Tripuraneni, Chi Jin, and Michael~I. Jordan.
\newblock Provable meta-learning of linear representations.
\newblock In {\em Proceedings of the 38th International Conference on Machine
  Learning}, 2021.

\bibitem{zhou2021meta}
Yufan Zhou, Zhenyi Wang, Jiayi Xian, Changyou Chen, and Jinhui Xu.
\newblock Meta-learning with neural tangent kernels.
\newblock In {\em Proceedings of the 9th International Conference on Learning
  Representations}, 2021.

\end{thebibliography}
\bibliographystyle{plain}

\appendix
\newpage
\section{Additional Related Work}
Data-driven algorithm selection is an algorithm design paradigm for setting algorithm parameters when multiple instances of a problem are available or need to be solved \cite{blum2020technical,balcan2020data}. It is familiar as {\it hyperparameter tuning} to machine learning practitioners which often involves a ``grid search'', ``random search'' or gradient-based search, with no formal guarantees of convergence to a global optimum. By modeling the problem of identifying a good algorithm from data as a statistical learning problem, general learning algorithms have been developed which exploit smoothness of the underlying algorithmic distribution \cite{balcan2018dispersion}. This provides a new algorithmic perspective, along with tools and insights for good performance under this smoothed analysis for fundamental problems including clustering, mechanism design, and mixed integer programs, and providing guarantees like differential privacy, adaptive online learning and adversarial robustness \cite{balcan2019learning,balcan2018general,balcan2020learning,balcan2020power}.

\section{Proofs}

\subsection{Proof of Theorem~\ref{thm:exp-forc-meta}}

\begin{proof}
The proof adapts the analysis of the exponential forecaster in \cite{balcan2018dispersion}. Let $W_t = \int_Cw_t(\rho) d\rho$ be
the normalizing constant and $P_t = \E_{\rho\sim p_t}
[u_t(\rho)]$ be the expected payoff at round $t$. Also let $U_t(\rho)=\sum_{j=1}^{t}u_j(\rho)$. We seek to bound $R_T=OPT-P(T)$, where $OPT=U_{T}(\rho^*)$ for optimal parameter $\rho^*$ and $P(T)=\sum_{t=1}^{T}P_t$ is the expected utility of Algorithm \ref{alg:ef} in $T$ rounds.
We will do this by lower bounding $P(T)$ and upper bounding $OPT$ by analyzing the normalizing constant  $W_t$.

{\it Lower bound for $P(T)$}: This follows from standard arguments, included for completeness. Using the definitions in Algorithm \ref{alg:ef}, it follows that
\begin{align*}\frac{W_{t+1}}{W_{t}} &= \frac{\int_{\C}e^{\lambda u_t(\rho)}w_{t}(\rho)d\rho}{W_{t}} = \int_{\C}e^{\lambda u_t(\rho)}\frac{w_{t}(\rho)}{W_{t}}d\rho = \int_{\C}e^{\lambda u_t(\rho)}p_{t}(\rho)d\rho.\end{align*}
Use inequalities $e^{\lambda x}\le1+(e^{\lambda}-1)x$ for $x\in[0,1]$ and $1+x\le e^x$ to conclude
\begin{align*}\frac{W_{t+1}}{W_{t}} \le\int_{\C}p_{t}(\rho)\left(1+(e^{\lambda}-1)u_t(\rho)\right)d\rho = 1+(e^{H\lambda}-1){P_t} \le \exp\left((e^{\lambda}-1){P_t}\right).\end{align*}
Finally, we can write $W_{T+1}/W_1$ as a telescoping product to obtain
\[\frac{W_{T+1}}{W_{1}}=\prod_{t=1}^{T}\frac{W_{t+1}}{W_{t}}\le \exp\left((e^{\lambda}-1){\sum_tP_t}\right) = \exp\left({P(T)(e^{\lambda}-1)}\right),\]
or, $W_{T+1}\le \exp\left({P(T)(e^{\lambda}-1)}\right)\int_Cw_1(\rho)d\rho$.


{\it Upper bound for $OPT$}: Let $\B^* (r)$ be the ball of radius $r$ around $\rho^*$. If there are at most $k$ discontinuities in any ball of radius $r$, we can conclude that for all $\rho\in\B^* (r)$, $U_{T}(\rho) \ge OPT - k-LTr$. Now, since $W_{T+1}=\int_Cw_1(\rho)\exp(\lambda U_{T}(\rho))d\rho$, we have

\begin{align*}
     W_{T+1}
&\ge \int_{\B^* (r)}w_1(\rho)e^{\lambda U_{T}(\rho)}d\rho\\&\ge \int_{\B^* (r)}w_1(\rho)e^{\lambda(OPT - k-LTr)}d\rho \\&=e^{\lambda(OPT - k-LTr)}\int_{\B^* (r)}w_1(\rho)d\rho.
\end{align*}

Putting together with the lower bound, and rearranging, gives
\begin{align*}OPT-P_T&\le \frac{P(T)(e^{\lambda}-1-\lambda)}{\lambda}+\frac{\log (1/Z)}{\lambda}+k+LTr\\
&\le T\lambda +\frac{\log (1/Z)}{\lambda}+k+LTr,\end{align*}
where we use that $P(T)\le T$ and for all $x\in[0,1], e^x \le 1 + x + (e-2)x^2$. Take expectation over the sequence of utility functions and apply dispersion to conclude the result.
\end{proof}
\subsection{Proof of Theorem~\ref{thm:dispersion-lb}}
We extend the construction in \cite{balcan2020learning} to the multi-task setting. The main difference is that we generalize the construction for any task similarity, and show that we get the same lower bound asymptotically.
\begin{proof}
    Define $u^{(b,x)}(\rho)=I[b=0]*I[\rho>x]+I[b=1]*I[\rho\le x]$, where $b\in\{0,1\}$, $x,\rho\in[0,1]$ and $I[\cdot]$ is the indicator function. For each iteration the adversary picks  $u^{(0,x)}$ or $u^{(1,x)}$ with equal probability for some $x\in [a,a+D^*]$, the ball of diameter $D^*$ containing all the optima. 
    
    For each task $t$, $m-\frac{3}{D^*}m^{1-\beta}$ functions are presented with the discontinuity $x\in [a+D^*/3,a+2D^*/3]$ while ensuring $\beta$-dispersion. The remaining $\frac{3}{D^*}m^{1-\beta}$ are presented with discontinuities located in successively halved intervals (the `halving adversary') containing the optima in hindsight, any algorithm gets half of these wrong in expectation. It is readily verified that the functions are $\beta$-dispersed. The construction works provided $m$ is sufficiently large ($m>\left(\frac{3}{D^*}\right)^{1/\beta}$). The task averaged regret is therefore also $\Tilde{\Omega}(m^{1-\beta})$.
\end{proof}

\subsection{Proof of Theorem~\ref{lem:aruba}}

\begin{proof}
	\begin{align*}
	\sum_{t=1}^T\sum_{m=1}^m
	\ell_{t,i}(\rho_{t,i})-\min_{\rho_t^\ast\in C}&\sum_{i=1}^m\ell_{t,i}(\rho_t^\ast)\\
	&\le\sum_{t=1}^TU_t(w_t,v_t)\\
	&\le\min_{v>0}H_T(v)\sqrt m+\sum_{t=1}^T\left(v+\frac{f_t(w_t)}v\right)\sqrt m+g(m)\\
	&\le\min_{w:C\mapsto\R_{\ge0},v>0}H_T(v)\sqrt m+\frac{F_T(w)\sqrt m}v+\sum_{t=1}^T\left(v+\frac{f_t(w)}v\right)\sqrt m+g(m)\\
	&\le\left(H_T(V)+\min\left\{\frac{F_T(w^\ast)}V,2\sqrt{F_T(w^\ast)T}\right\}+2TV\right)\sqrt m+Tg(m)
	\end{align*}
	where the last step is achieved by substituting $w=w^\ast$ and $v=\max\left\{V,\sqrt{F_T(w^\ast)/T}\right\}$.
\end{proof}

\subsection{Proof of Lemma~\ref{lem:equivalent}}

\begin{proof}
	Define a probability measure $p:C\mapsto\R_{\ge0}$ that is constant on all elements $\tilde D\in\D_t$ of the discretization at time $t$, taking the value $p(\rho)=\frac1{\vol(\tilde D)}\sum_{D\in D_T,D\subset\tilde D}\*w_{[D]}~\forall~\rho\in\tilde D$.
	Note that for any $D\in\D_T$ that is a subset of $\tilde D$ we have that 
	$$\*p_{[D]}=\int_D\tilde w(\rho)d\rho=\frac{\*v_{[D]}}{\sum_{D'\in\D_T,D'\subset\tilde D}\*v_{[D']}}\sum_{D'\in\D_T,D'\subset\tilde D}\*w_{[D']}$$
	Then
	\begin{align*}
	D_{KL}&(\*p||\*{\hat v})-\eta\sum_{s\le t}\log\langle\*w_s^\ast,\*p\rangle\\
	&=\sum_{\tilde D\in\D_t}\sum_{D\in\D_T,D\subset\tilde D}\*p_{[D]}\log\frac{\*p_{[D]}}{\*{\hat v}_{[D]}}-\eta\sum_{s\le t}\log\sum_{\tilde D\in\D_t}\sum_{D\in\D_T,D\subset\tilde D}{\*w_s^\ast}_{[D]}\*p_{[D]}\\
	&=\sum_{\tilde D\in\D_t}\sum_{D\in\D_T,D\subset\tilde D}\frac{\*v_{[D]}}{\sum_{D'\in\D_T,D'\subset\tilde D}\*v_{[D']}}\sum_{D'\in\D_T,D'\subset\tilde D}\*w_{[D']}\log\frac{\sum_{D'\in\D_T,D'\subset\tilde D}\*w_{[D']}}{\sum_{D'\in\D_T,D'\subset\tilde D}\*{\hat v}_{[D']}}\\
	&\quad-\eta\sum_{s\le t}\log\sum_{\tilde D\in\D_t}\sum_{D\in\D_T,D\subset\tilde D}\frac{{\*w_s^\ast}_{[D]}\*v_{[D]}}{\sum_{D'\in\D_T,D'\subset\tilde D}\*v_{[D']}}\sum_{D'\in\D_T,D'\subset\tilde D}\*w_{[D']}\\
	&\le\sum_{\tilde D\in\D_t}\sum_{D\in\D_T,D\subset\tilde D}\frac{\*v_{[D]}}{\sum_{D'\in\D_T,D'\in\tilde D}\*v_{[D']}}\sum_{D'\in\D_T,D'\subset\tilde D}\*w_{D'}\log\frac{\*w_{[D']}}{\*{\hat v}_{[D']}}\\
	&\quad-\eta\sum_{s\le t}\log\sum_{\tilde D\in\D_t,\tilde D\subset C_s}\sum_{D\in\D_T,D\subset\tilde D}\frac{\*v_{[D]}}{\sum_{D'\in\D_T,D'\subset\tilde D}\*v_{[D']}}\sum_{D'\in\D_T,D'\subset\tilde D}\*w_{[D']}\\
	&=\sum_{\tilde D\in\D_t}\sum_{D'\in\D_T,D'\in\tilde D}\*w_{[D']}\log\frac{\*w_{[D']}}{\*{\hat v}_{[D']}}-\eta\sum_{s\le t}\log\sum_{\tilde D\in\D_t,\tilde D\subset C_s}\sum_{D'\in\D_T,D'\subset\tilde D}\*w_{[D']}\\
	&=D_{KL}(\*w||\*{\hat v})-\eta\sum_{s\le t}\log\langle\*w_s^\ast,\*w\rangle
	\end{align*}
	where the inequality follows from applying the log-sum inequality to the first term and the fact that ${\*w_s^\ast}_{[D]}=\1_{D\subset C_s}$ in the second term.
	Note that we also have
	$$\|\*p\|_1
	=\sum_{\tilde D\in\D_t}\sum_{D\in\D_T,D\subset\tilde D}\frac{\*v_{[D]}}{\sum_{D'\in\D_T,D'\subset\tilde D}\*v_{[D']}}\sum_{D'\in\D_T,D'\subset\tilde D}\*w_{[D']}
	=\sum_{\tilde D\in\D_t}\sum_{D'\in\D_T,D'\subset \tilde D}\*w_{[D']}
	=1$$
	and 
	$$\*p_{[D]}
	=\frac{\*v_{[D]}}{\sum_{D'\in\D_T,D'\subset\tilde D}\*v_{[D']}}\sum_{D'\in\D_T,D'\subset\tilde D}\*w_{[D']}
	\ge\frac{\gamma\*v_{[D]}}{\sum_{D'\in\D_T,D'\subset\tilde D}\*v_{[D']}}\sum_{D'\in\D_T,D'\subset\tilde D}\*{\hat v}_{[D']}
	=\gamma\*{\hat v}_{[D]}$$
	so $\*p$ satisfies the optimization constraints.
	Therefore, since $\*w$ was defined to be the minimum of the sum of the KL-divergence (a strongly-convex function \cite[Example~2.5]{shalev-shwartz2011oco}) and a convex function, it is unique and so coincides with $\*p$.
	
	On the other hand
	\begin{align*}
	D_{KL}(\*p(t)||\*{\hat v(t)})-\eta\sum_{s\le t}\log\langle\*w_s^\ast(t),\*p(t)\rangle
	&\le D_{KL}(\*p||\*{\hat v})-\eta\sum_{s\le t}\log\langle\*w_s^\ast,\*p\rangle\\
	&=D_{KL}(\*w||\*{\hat v})-\eta\sum_{s\le t}\log\langle\*w_s^\ast,\*w\rangle\\
	&\le D_{KL}(\*{\tilde w}||\*{\hat v})-\eta\sum_{s\le t}\log\langle\*w_s^\ast,\*{\tilde w}\rangle\\
	&=D_{KL}(\*{\tilde w}(t)||\*{\hat v}(t))-\eta\sum_{s\le t}\log\langle\*w_s^\ast(t),\*{\tilde w}(t)\rangle
	\end{align*}
	where the first inequality follows from above and the second from the optimality of $\*w$.
	Note that by nonnegativity the discretization of $p$ does not affect its measure over $C$, so $\|\*p\|_1=1\implies\|\*p(t)\|_1=1$.
	Finally, also from above we have 
	$$\*p(t)_{[D]}=\sum_{D'\in\D_T,D'\subset D}\*p_{[D']}\ge\gamma\sum_{D'\in\D_T,D'\subset D}\*p_{[D']}\*{\hat v}_{[D']}=\gamma\*{\hat v}(t)_{[D]}$$
	Thus as before $\*p(t)$ satisfies the optimization constraints, which with the previous inequality and the uniqueness of the optimum $\*{\tilde w}(t)$ implies that $\*p(t)=\*{\tilde w}(t)$.
	Finally, since $\tilde w$ is constant on all elements of the discretization $\D_t$ of $C$ this last fact implies that $\*p=\*{\tilde w}$, which together with $\*p=\*w$ implies the result.
\end{proof}

\subsection{Lipschitzness for Algorithm~\ref{alg:ftrl}}

\begin{Clm}\label{clm:overlip}
	The loss $f_t$ is $\frac1{\gamma\vol(C_t)}$-Lipschitz w.r.t. $\|\cdot\|_1$ over the set $\{\*w\in\R^{|\D_T|}:\|\*w\|_1=1,\*w\ge\gamma\*{\hat v}\}$.
\end{Clm}
\begin{proof}
	$$\max_{\|\*w\|_1=1,\*w\ge\gamma\*{\hat v}}\|\nabla\log\langle\*w_t^\ast,\*w\rangle\|_\infty
	=\max_{D,\|\*w\|_1=1,\*w\ge\gamma\*{\hat v}}\frac{{\*w_t^\ast}_{[D]}}{\langle\*w_t^\ast,\*w\rangle}
	\le\frac1{\langle\*w_t^\ast,\gamma\*{\hat v}\rangle}
	=\frac1{\gamma\vol(C_t)}$$
\end{proof}

\subsection{Proof of Corollary~\ref{cor:ewoo}}
\begin{proof}
	Using first-order conditions we have that the optimum in hindsight of the functions $h_t$ satisfies
	$$v^2
	=\frac1T\sum_{t=1}^Tf_t(w_t)
	=-\frac1T\sum_{t=1}^T\log\langle\*w_t^\ast,\*w_t\rangle
	\le\frac1T\sum_{t=1}^T\log\frac1{\gamma\vol(C_t)}$$
	Applying \cite[Corollary~C.2]{khodak2019adaptive} with $\alpha_t=1$, $B_t^2=f_t(w_t)$, and $D^2-\log\gamma$ instead of $D^2$ yields the result.
\end{proof}

\subsection{Proof of Corollary~\ref{cor:ftl}}
\begin{proof}
	Using first-order conditions we have that the optimum in hindsight of the functions $h_t$ satisfies
	$$v^2
	=\frac1T\sum_{t=1}^Tf_t(w_t)
	=-\frac1T\sum_{t=1}^T\log\langle\*w_t^\ast,\*w_t\rangle
	\le\frac1T\sum_{t=1}^T\log\frac1{\gamma\vol(C_t)}$$
	Applying \cite[Proposition~B.2]{khodak2019adaptive} with $\alpha_t=1$, $B_t^2=f_t(w_t)$, and $D^2-\log\gamma$ instead of $D^2$ yields the result.
\end{proof}

\subsection{Proof of Theorem~\ref{thm:tar}}

\begin{proof}
	We have $F_T(w^\ast)=\tilde O(\sqrt {BG}T^\frac34)$ and $H_T(V)=\tilde O(\min\{1/V,\sqrt[5]T\}T^\frac35)$ from Corollaries~\ref{cor:ewoo} and~\ref{cor:ftl}.
	Substituting into Lemma~\ref{lem:equivalent} and simplifying yields
	$$\tilde O\left(\frac{\min\left\{\frac1V,\sqrt[4]T\right\}}{\sqrt T}+\min\left\{\frac{\sqrt{BG}}{V\sqrt[4]T},\frac{\sqrt[4]{BG}}{\sqrt[8]T}\right\}+2V\right)\sqrt m+g(m)$$
	Simplifying further yields the result.
\end{proof}

\subsection{Proof of Theorem \ref{thm:robustness single task}}

\begin{proof}
The bound on $\E[\Tilde{R}_T]$ is immediate from Theorem \ref{thm:exp-forc-meta}. For $\E[R_T]$, we can upper bound the natural regret with the sum of robust regret, total adversarial perturbation at the optimum and a term corresponding to the difference between the loss of natural and robust optima. 
\begin{align*}
    R_T &=\sum_{t=1}^T l_t(x_t) - \min_{x\in\X}\sum_{t=1}^T l_t(x)\\
    &=\Tilde{R}_T+\sum_{t=1}^T l_t(x_t) - \sum_{t=1}^T \Tilde{l}_t(x_t) + \min_{x\in\X}\sum_{t=1}^T \Tilde{l}_t(x) - \min_{x\in\X}\sum_{t=1}^T l_t(x) \\
    &=\Tilde{R}_T-\sum_{t=1}^T a_t(x_t) + \sum_{t=1}^Ta_t(\Tilde{x}^*)+ \sum_{t=1}^Tl_t(\Tilde{x}^*) - \sum_{t=1}^Tl_t(x^*)\\
    &\le \Tilde{R}_T + \sum_{t=1}^Ta_t(\Tilde{x}^*)+\Big\lvert \sum_{t=1}^Tl_t(\Tilde{x}^*) - \sum_{t=1}^Tl_t(x^*)\Big\rvert
\end{align*}
where $\Tilde{x}^* = \argmin_{x\in\X}\sum_{t=1}^T \Tilde{l}_t(x)$ and $x^* = \argmin_{x\in\X}\sum_{t=1}^T l_t(x)$. We now use the $\beta_a$-dispersedness of the attack to show an excess expected regret of $\Tilde{O}(T^{1-\beta_a})$.
Using attack dispersion on a ball of radius $T^{-\beta_a}$ around $\Tilde{x}^*$, the number of attacks that have non-zero $a_t(\Tilde{x}^*)$ is at most $\Tilde{O}(T^{1-\beta_a})$, and therefore $\sum_{t=1}^Ta_t(\Tilde{x}^*)\le \Tilde{O}(T^{1-\beta_a})$. Further, observe that the robust and natural optima coincide unless some attack occurs at the natural optimum $x^*$. 
We can use attack dispersion at $x^*$, and a union bound across rounds, to conclude $\E\lvert \sum_{t=1}^Tl_t(\Tilde{x}^*) - \sum_{t=1}^Tl_t(x^*)\rvert\le\Tilde{O}(T^{1-\beta_a})$ which concludes the proof.
\end{proof}

\subsection{Proof of Theorem \ref{thm:robustness lower bound}}

\begin{proof}
Part 1 follows from the lower bound in Theorem \ref{thm:dispersion-lb}, by setting $\Tilde{l}_i=l_i$ as the loss sequence used in the proof.

To establish Part 2, we extend the construction as follows. $\Tilde{l}_i=l_i$ are both equal and correspond to the `halving adversary' from the proof of Theorem \ref{thm:dispersion-lb} for the first $\Theta(m^{1-\beta})$ rounds. If $\beta\le\beta_a$ we are done, so assume otherwise. Let $I$ denote the interval containing the optima over the rounds so far. Notice that the length of $I$ is at most $|I|\le (\frac{1}{2})^{\Theta(m^{1-\beta})}\le (\frac{1}{2})^{\beta \log m}=m^{-\beta}$ for $\beta>0$. For further rounds $l_i$ continues to be the halving adversary for $\Theta(m^{1-\beta_a})$ rounds, which implies any algorithm suffers $\Omega(m^{1-\beta_a})$ regret. We set attack $a_i$ on interval $I$ such that $\Tilde{l}_i=0$ on $I$ on these rounds. This ensures that $a_i$ is $\beta_a$-dispersed and $\Tilde{l}_i$ is $\beta$-dispersed. Putting together with the case $\beta\le\beta_a$, we obtain $\Omega(m^{1-\min\{\beta,\beta_a\}})$ bound on the regret of any algorithm.
\end{proof}

\section{Learning algorithmic parameters for combinatorial problems}\label{app: combinatorial}
We discuss implications of our results for several combinatorial problems of widespread interest including integer quadratic programming and auction mechanism design. We will need the following theorem from \cite{balcan2021data}, which generalizes the recipe for establishing dispersion given by \cite{balcan2020semi} for $d=1,2$ dimensions to arbitrary constant $d$ dimendions. It is straightforward to apply the recipe to establish dispersion for these problems, which in turn implies that our meta-learning results are applicable. We demonstrate this for a few important problems below for completeness.

\begin{Thm}[\cite{balcan2021data}]\label{thm:dispersion-recipe}
Let $l_1, \dots, l_m : \R^d \rightarrow \R$ be independent piecewise $L$-Lipschitz functions, each having discontinuities specified by a collection of at most $K$ algebraic hypersurfaces of bounded degree. Let $\mathcal{L}$ denote the set of axis-aligned paths between pairs of points in $\R^d$, and for each $s\in \mathcal{L}$ define
 $D(m, s) = |\{1 \le t \le m \mid l_t\text{ has a discontinuity along }s\}|$. Then we
have $\E[\sup_{s\in \mathcal{L}} D(m, s)] \le \sup_{s\in \mathcal{L}} \E[D(m, s)] +
O(\sqrt{m \log(mK)})$.
\end{Thm}

\subsection{Greedy knapsack}
We are given a knapsack with capacity $\texttt{cap}$ and items $i\in[m]$ with sizes $w_i$ and values $v_i$. The goal is to select a subset $S$ of items to add to the knapsack such that $\sum_{i\in S}w_i\le \texttt{cap}$ while maximizing the total value $\sum_{i\in S}v_i$ of selected items. We consider a general greedy heuristic to insert items with largest $v_i/w_i^{\rho}$ first (due to \cite{gupta2017pac}) for $\rho\in[0,10]$. 

The classic greedy heuristic sets $\rho=1$ and can be used to provide a 2-approximation for the problem. However other values of $\rho$ can improve the knapsack objective on certain problem instances. For example,
for the value-weight pairs $\{(0.99,1),(0.99,1),(1.01,1.01)\}$ and capacity $\texttt{cap}=2$ the classic heuristic $\rho=1$ gives value $1.01$ as the greedy heuristic is maximized for the third item. However, using $\rho=3$ (or any $\rho>1+\log(1/0.99)/\log(1.01)>2.01$) allows us to pack the two smaller items giving the optimal value $1.98$.

Our result (Theorem \ref{thm:tar}) when applied to this problem shows that it is possible to learn the optimal parameter values for the greedy heuristic algorithm family for knapsack from similar tasks.

\begin{Thm}
Consider instances of the knapsack problem given by bounded weights $w_{i,j}\in[1,C]$ and $\kappa$-bounded independent values $v_{i,j}\in[0,1]$ for $i\in[m],j\in[T]$. Then the asymptotic task-averaged regret for learning the algorithm parameter $\rho$ for the greedy heuristic family described above is $o_T(1)+2V\sqrt{m}+O(\sqrt{m})$.
\end{Thm}

\begin{proof}
Lemma 11 of \cite{balcan2020semi} shows that the loss functions form a $\frac{1}{2}$-dispersed sequence. The result follows by applying Theorem \ref{thm:tar} with $\beta=\frac{1}{2}$.
\end{proof}

\subsection{$k$-center clustering} We consider the $\alpha$-Lloyd's clustering algorithm family from \cite{balcan2018data}, where the initial $k$ centers in the procedure are set by sampling points with probability  proportional to $d^\alpha$ where $d$ is the distance from the centers selected so far for some $\alpha\in[0,D],D\in\R_{\ge0}$. For example, $\alpha=0$ corresponds to the vanilla $k$-means with random initial centers, and $\alpha=2$ setting is the $k$-means++ procedure. For this algorithm family, we are able to show the following guarantee. Interestingly, for this family it is sufficient to rely on the internal randomness of the algorithmic procedure and we do not need assumptions on data smoothness.

\begin{Thm}
Consider instances of the $k$-center clustering problem on $n$ points, with Hamming loss $l_{i,j}$ for $i\in[m],j\in[T]$ against some (unknown) ground truth clustering. Then the asymptotic task-averaged regret for learning the algorithm parameter $\alpha$ for the $\alpha$-Lloyd's clustering algorithm family of \cite{balcan2018data} is $o_T(1)+2V\sqrt{m}+O(\sqrt{m})$.
\end{Thm}

\begin{proof}
We start by applying Theorem 4 from \cite{balcan2018data} to an arbitrary $\alpha$-interval $[\alpha_0,\alpha_0+\epsilon]\subseteq[0,D]$ of length $\epsilon$. The expected number of discontinuities (expectation under the internal randomness of the algorithm when sampling successive centers), is at most
$$D(m,\epsilon)=O(nk \log(n) \log(\max\{(\alpha_0+\epsilon)/\alpha_0),(\alpha_0+\epsilon)\log R\}),$$
where $R$ is an upper bound on the
ratio between any pair of non-zero distances. Considering cases $\alpha_0\lessgtr\frac{1}{\log R}$ and using the inequality $\log(1+x)\le x$ for $x\ge 0$ we get that there are, in expectation, at most $O(\epsilon nk \log n \log R)$ discontinuities in any interval of length $\epsilon$. Theorem \ref{thm:dispersion-recipe} now implies $\frac{1}{2}$-dispersion using the recipe from \cite{balcan2020semi}. The task-averaged regret bound follows from Theorem \ref{thm:tar}.
\end{proof}

\subsection{Integer quadratic programming (IQP)} The objective is to maximize a quadratic function $z^TAz$ for $A$ with non-negative diagonal entries, subject to $z\in\{0,1\}^n$. In the classic Goemans-Williamson algorithm \cite{goemans1995improved} one solves an SDP relaxation $U^TAU$ where columns $u_i$ of $U$ are unit vectors. $u_i$ are then rounded to $\{\pm 1\}$ by projecting on a vector $Z$ drawn according to the standard Gaussian, and using $\texttt{sgn}(\langle u_i,Z\rangle)$. A simple parametric family is $s$-linear rounding where the rounding is as before if $|\langle u_i,Z\rangle|>s$ but uses probabilistic rounding to round $u_i$ to 1 with probability $\frac{1+(\langle u_i,Z\rangle)/s}{2}$. The dispersion analysis of the problem from \cite{balcan2018dispersion} and the general recipe from \cite{balcan2020semi} imply that our results yield low task-averaged regret for learning the parameter of the $s$-linear rounding algorithms.

\begin{Thm}
Consider instances of IQP given by matrices $A_{i,j}$ and rounding vectors $Z_{i,j}\sim \mathcal{N}_n$ for $i\in[m],j\in[T]$. Then the asymptotic task-averaged regret for learning the algorithm parameter $s$ for $s$-linear rounding is $o_T(1)+2V\sqrt{m}+O(\sqrt{m})$.
\end{Thm}

\begin{proof}
As noted in \cite{balcan2018dispersion}, since $Z_{i,j}$ are normal, the local of discontinuities $s=|\langle u_i,Z\rangle|$ are distributed with a $\sqrt{\frac{2}{\pi}}$-bounded density. Thus in any interval of length $\epsilon$, we have in expectation at most $\epsilon\sqrt{\frac{2}{\pi}}$ discontinuities. Theorem \ref{thm:dispersion-recipe} together with the general recipe from \cite{balcan2020semi} implies $\frac{1}{2}$-dispersion. The task-averaged regret bound is now a simple application of Theorem \ref{thm:tar}.
\end{proof}

Our results are an improvement over prior work which have only considered iid and (single-task) online learning settings. Similar improvements can be obtained for auction design, as described below. We illustrate this using a relatively simple auction, but the same idea applies for an extensive classes of auctions as studied in \cite{balcan2018general}.

\subsection{Posted price mechanisms with additive valuations} There are $m$ items and $n$ bidders with valuations $v_j(b_i),j\in[n],i\in[2^m]$ for all $2^m$ bundles of items. We consider additive valuations which satisfy $v_j(b)=\sum_{i\in b}v_j(\{i\})$. The objective is to maximize the social welfare (sum of buyer valuations). If the item values for each buyer have $\kappa$-bounded distributions, then the corresponding social welfare is dispersed and our results apply.

\begin{Thm}
Consider instances of posted price mechanism design problems with additive buyers and $\kappa$-bounded marginals of item valuations. Then the asymptotic task-averaged regret for learning the price which maximizes the social welfare is $o_T(1)+2V\sqrt{m}+O(\sqrt{m})$.
\end{Thm}

\begin{proof}
As noted in \cite{balcan2018dispersion}, the locations of discontinuities are along axis-parallel hyperplanes (buyer
$j$ will be willing to buy item $i$ at a price 
$p_i$
if and only if $v_j(\{i\}) \ge p_i$, each buyer-item pair in each instance corresponds to a hyperplane). Thus in any pair of points $p,p'$ (corresponding to pricing) at distance $\epsilon$, we have in expectation at most $\epsilon\kappa mn$ discontinuities along any axis-aligned path joining $p,p'$, since discontinuities for an item can only occur along axis-aligned segment for the axis corresponding to the item. Theorem \ref{thm:dispersion-recipe} now implies $\frac{1}{2}$-dispersion. The task-averaged regret bound is now a simple application of Theorem \ref{thm:tar}.
\end{proof}

\section{Additional experiments}\label{app: experiment}


\subsection{Number of training tasks needed for meta-learning}

We also examine the number of training tasks that our meta-learning procedure needs to obtain improvements over the single-task baseline. We use a single test task, and a variable number of training tasks (0 through 10) to meta-learn the initialization. We use the same settings as in Section \ref{sec:experiments}, except the meta-learning experiments have been averaged over 20 iterations (to average over randomization in the algorithms). In Figure \ref{fig: regret vs meta-updates}, we plot the average regret against number of meta-updates performed before starting the test task, and compare against the single-task baselines. We observe gains with meta-learning with just $T=10$ tasks for the Omniglot dataset, and with even a single task in the Gaussian mixture dataset. The latter is likely due to a very high degree of task similarity across all the tasks (examined below), so learning on any task transfers very well to another task.

\begin{figure}[!h]
    \centering
    \begin{subfigure}[b]{0.4\textwidth}
    \centering
         \includegraphics[width=\textwidth]{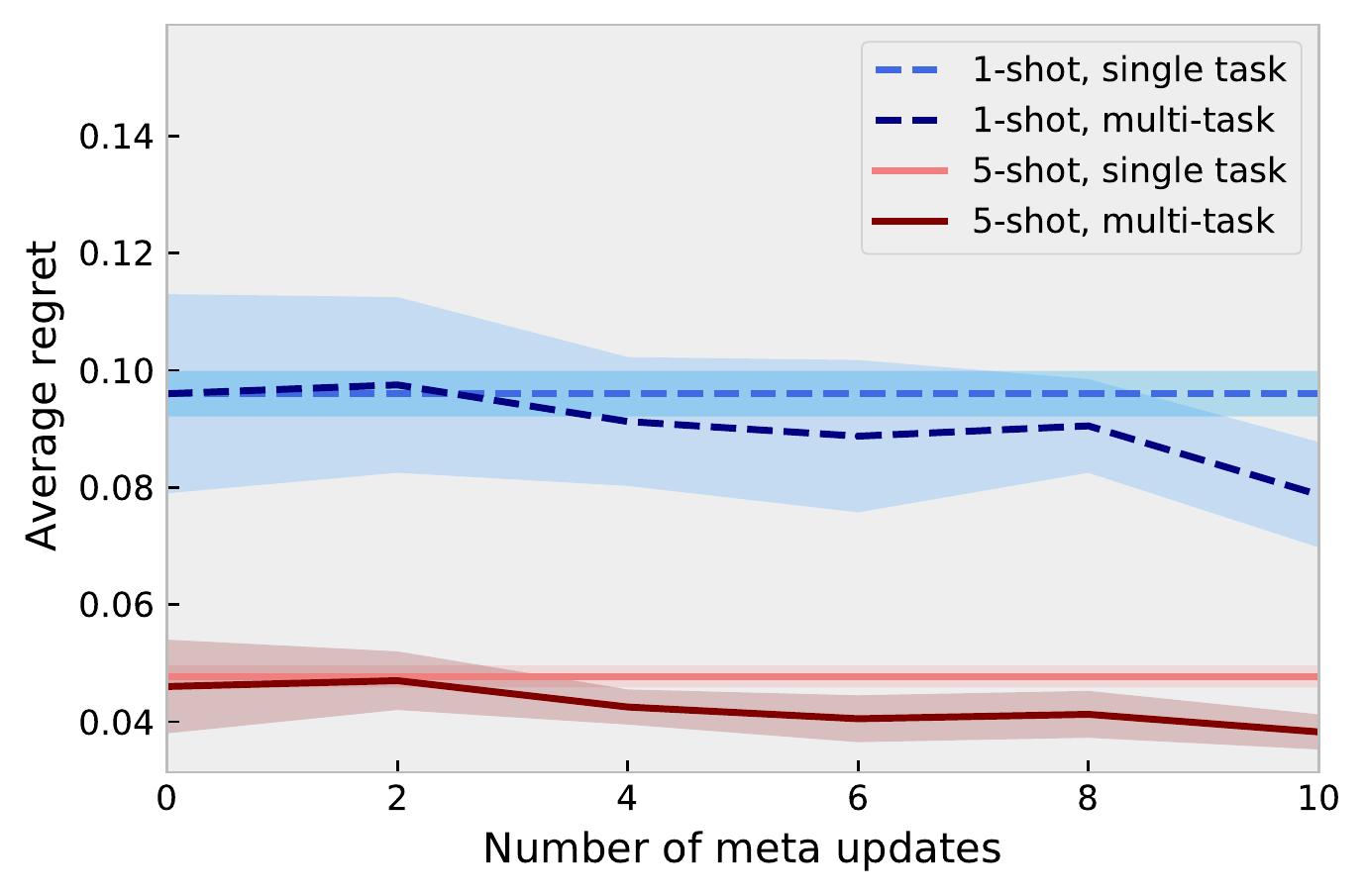}
    \caption{Omniglot}
  \end{subfigure}
    \begin{subfigure}[b]{0.4\textwidth}
    \centering
         \includegraphics[width=\textwidth]{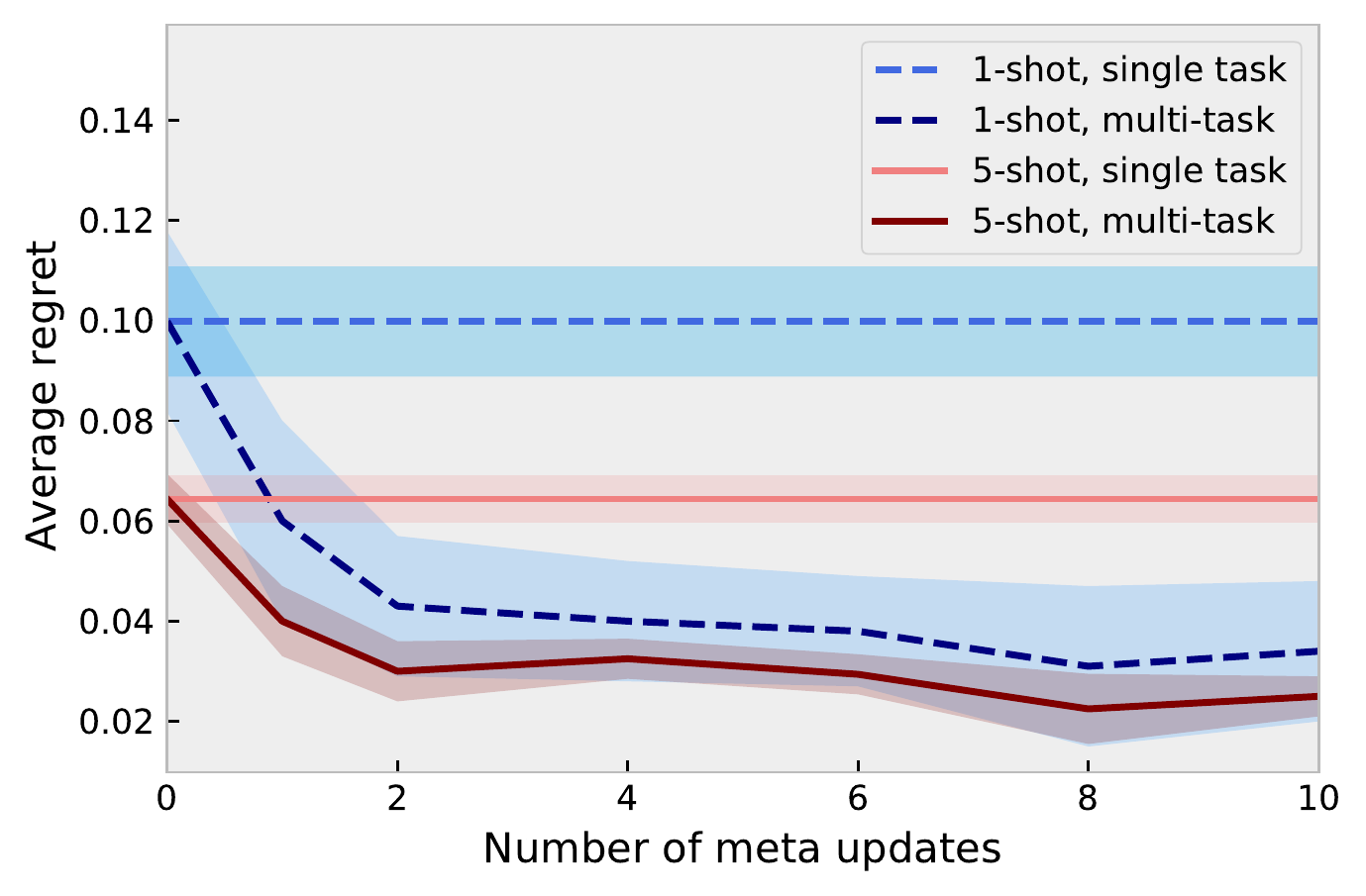}
    \caption{Gaussian mixture}
  \end{subfigure}
  \caption{Average regret vs. number of training tasks for meta-learning.}
\label{fig: regret vs meta-updates}
\end{figure}

\subsection{Task similarity and dispersion}

We also examine the task similarity of the different tasks by plotting the optimal values $\alpha^*_t$ of the clustering parameter $\alpha$ and the corresponding balls $\B(\alpha_t^\ast,m^{-\beta})$ used in our definition of task similarity (Figure \ref{fig: task similarity}). 

\begin{figure}[!h]
    \centering
    \begin{subfigure}[b]{0.35\textwidth}
    \centering
         \includegraphics[width=\textwidth]{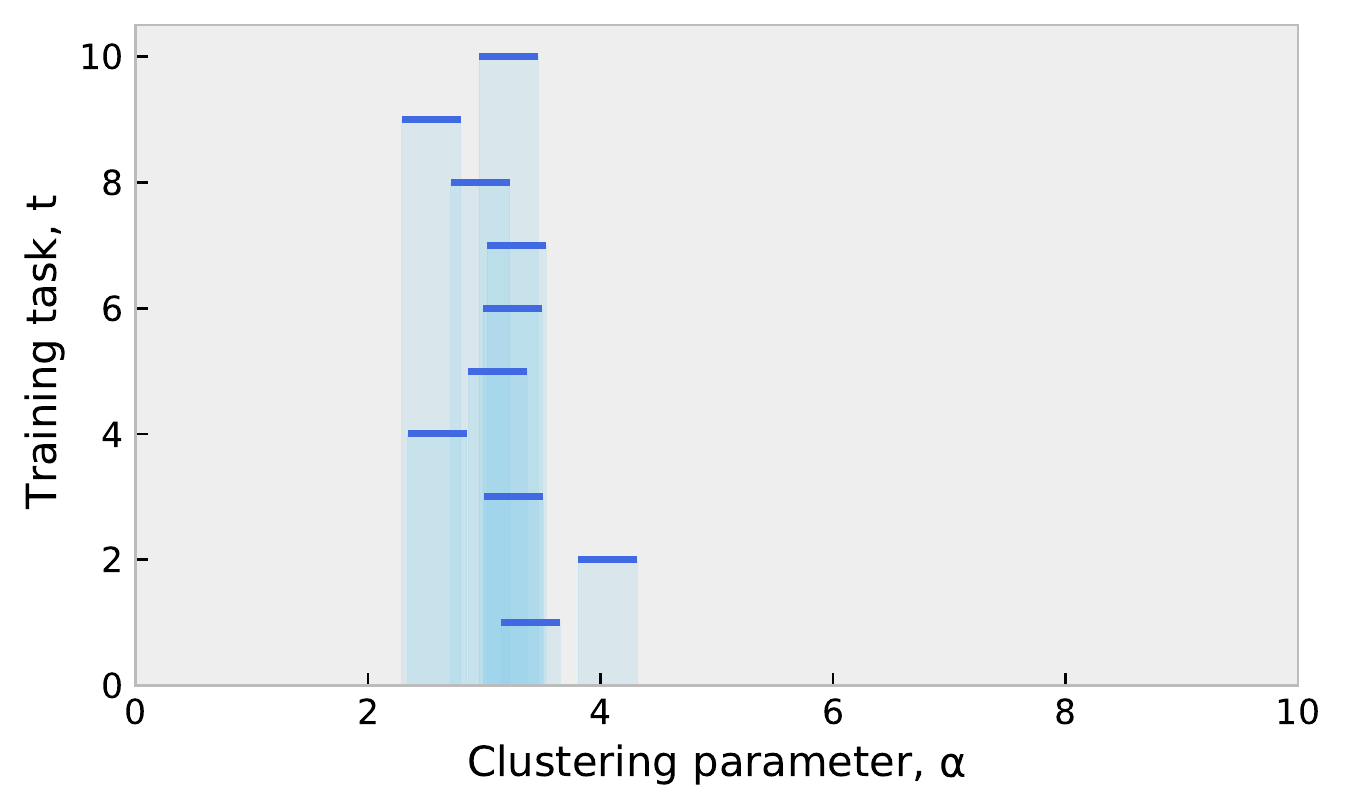}
    \caption{Omniglot}
  \end{subfigure}
    \begin{subfigure}[b]{0.35\textwidth}
    \centering
         \includegraphics[width=\textwidth]{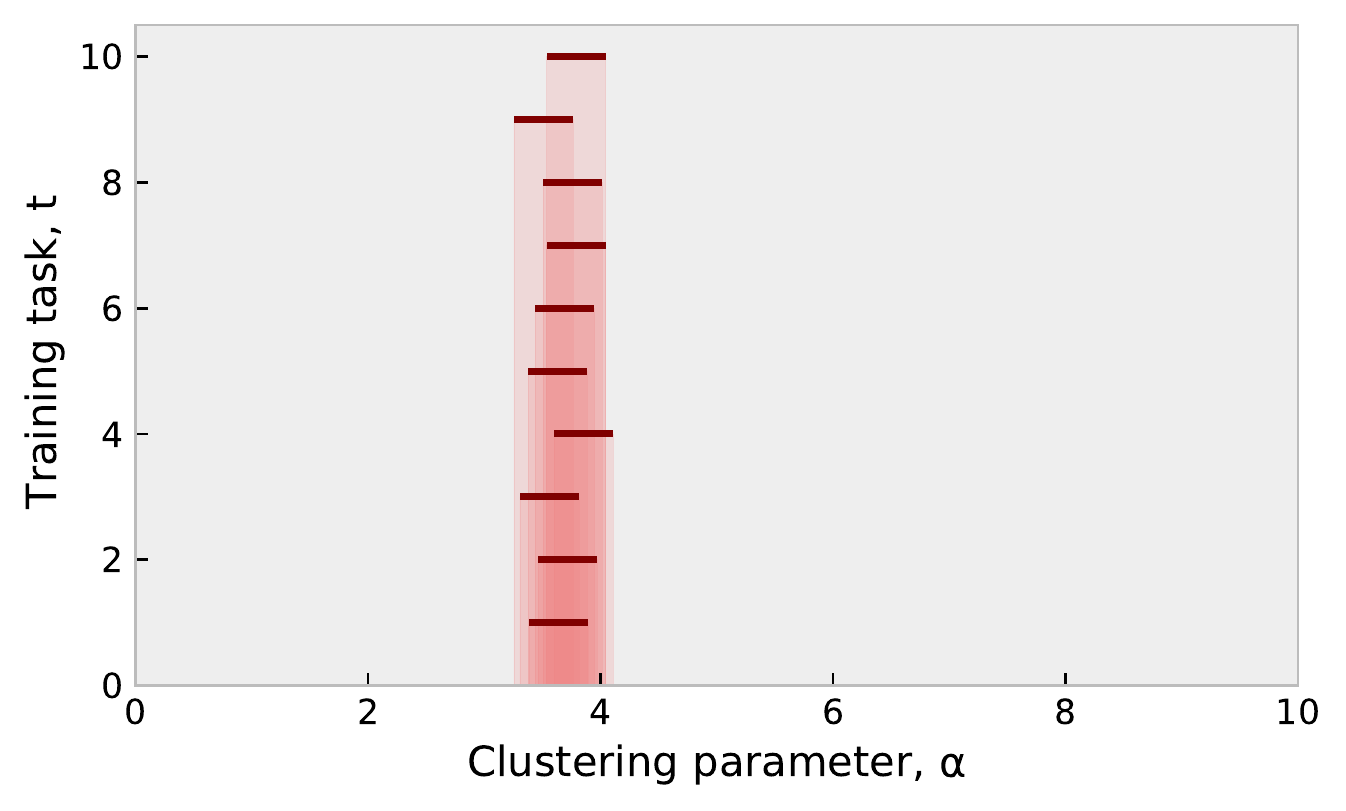}
    \caption{Gaussian mixture}
  \end{subfigure}
    \begin{subfigure}[b]{0.35\textwidth}
    \centering
         \includegraphics[width=\textwidth]{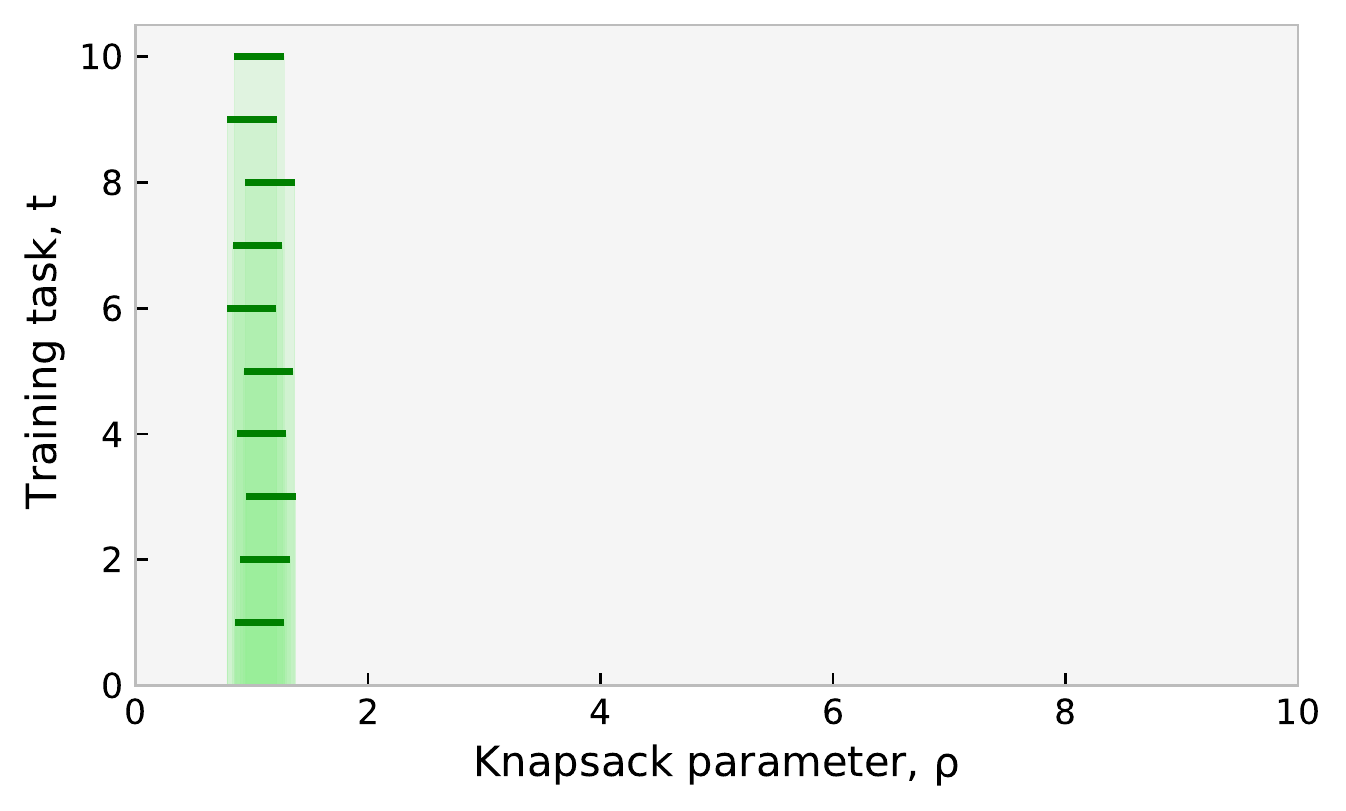}
    \caption{Knapsack}
  \end{subfigure}
  \caption{Location of optimal parameter values for the training tasks.}
\label{fig: task similarity}
\end{figure}

The intervals of the parameter induced by these balls correspond to the discretization used by Algorithm \ref{alg:ftrl}. We notice a stronger correlation in task similarity for the Gaussian mixture clustering tasks, which implies that meta-learning is more effective here (both in terms of learning test tasks faster, and with lower regret). For knapsack the task similarity is also high, but it turns out that for our dataset there are very `sharp peaks' at the optima of the total knapsack values as a function of the parameter $\rho$. So even though meta-learning helps us get within a small ball of the optima, a few steps are still needed to converge and we do not see the single-shot benefits of meta-learning as we do for the Gaussian clustering experiment. 

\begin{figure}[!h]
    \centering
    \begin{subfigure}[b]{0.4\textwidth}
    \centering
         \includegraphics[width=\textwidth]{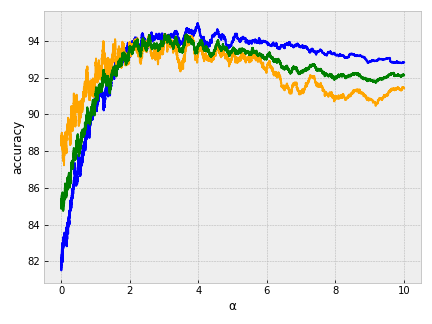}
    \caption{Clustering (Gaussian mixture dataset)}
  \end{subfigure}
    \begin{subfigure}[b]{0.4\textwidth}
    \centering
         \includegraphics[width=\textwidth]{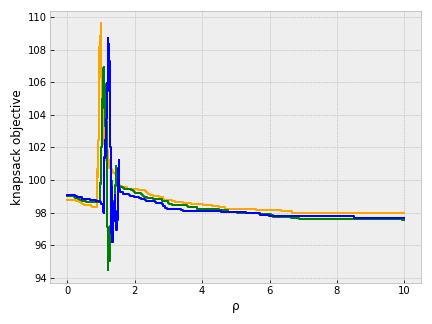}
    \caption{Greedy Knapsack}
  \end{subfigure}
  \caption{Average performance (over algorithm randomization) for a few tasks as a function of the configuration parameter. This explains why, despite high task similarity in either case, few-shot meta-learning works better for the Gaussian mixture clustering.}
\label{fig: average performance}
\end{figure}

\end{document}